%% file: RR.tex
\documentclass[11pt,a4paper,twoside]{paper}
\usepackage[latin1]{inputenc}
\usepackage{amsmath,amssymb,amsthm,bbold}
\usepackage{graphicx}
\usepackage{subfig}
\usepackage{color}
\usepackage{parskip}
\paragraphfont{\sffamily\itshape}

\newtheoremstyle{etoile}{\parskip}{\parskip}{\itshape}
                        {0pt}{\bfseries\sffamily}{.}{ }{}
\theoremstyle{etoile}

\newtheorem{prop}{Proposition}[section]

\makeatletter
\@addtoreset{equation}{section}
\makeatother

\newcommand\egaldef{\stackrel{\mbox{\upshape\tiny def}}{=}}

\newcommand\1{\leavevmode\hbox{\rm \small1\kern-0.35em\normalsize1}}

\newcommand{\be}{\begin{equation}}
\newcommand{\ee}{\end{equation}}
\newcommand{\bea}{\begin{eqnarray}}
\newcommand{\eea}{\end{eqnarray}}
\def\DD{\displaystyle}

\DeclareMathOperator*{\argmin}{\text{argmin}}
\DeclareMathOperator*{\argmax}{\text{argmax}}
\DeclareMathOperator*{\equ}{\approx}
\DeclareMathOperator*{\eq}{\sim}

\def\RR{{\rm I\hspace{-0.50ex}R}}

\def\c{{\mbox{\bf c}}}\index{\footnote{}}
\def\r{{\mbox{\bf r}}}

\def\HAP{\mbox{\sc Hi-AP}}
\def\AP{\mbox{\sc AP}}
\def\WAP{\mbox{\sc WAP}}
\def\RAP{\mbox{\sc RAP}}
\def\SCAP{\mbox{\sc SCAP}}
\def\DAC{{Divide-and-Conquer}}
\def\Pch{\mbox{$P_{\c(h)}$}}
\def\Pc{\mbox{$P_{\c}$}}
\def\rcm{{\mbox{$\r_{cm}$}}}
\def\rsd{{\mbox{$\r_{sd}$}}}
\def\rc{{\mbox{\bf r$_{\c}$}}}
\def\rt{{\mbox{$\tilde\r_{\c}$}}}
\def\tx{{\mbox{$x_{\widetilde{ex}}$}}}
\def\sp{{\mbox{$s_*$}}}

\def\Ga{{\mbox{$\Gamma\bigl(\frac{2}{d}\bigr)$}}}
\def\Gb{{\mbox{$\Gamma\bigl(1+\frac{2}{d}\bigr)$}}}

\title{Scaling Analysis of  Affinity Propagation}

\author{
Cyril Furtlehner\thanks{INRIA-Saclay, LRI B\^at. 490, F-91405 Orsay(France)}%
\and
MichËle Sebag\thanks{CNRS, LRI B\^at. 490, F-91405 Orsay(France)}%
\and
Xiangliang Zhang\thanks{INRIA, UniversitÈ Paris Sud,LRI B\^at. 490, F-91405 Orsay(France)}%
}

\begin{document}

\maketitle

\abstract{
We analyze and exploit some scaling properties of 
the {\em Affinity Propagation} (AP) clustering algorithm proposed by
Frey and Dueck (2007). First we observe that a divide and conquer strategy,  used on a large data set
hierarchically reduces the complexity ${\cal O}(N^2)$ to 
${\cal O}(N^{(h+2)/(h+1)})$, for a data-set of size $N$ and a depth $h$ 
of the hierarchical strategy. For a data-set embedded in a $d$-dimensional space, we show 
that this is obtained  without notably damaging the precision
except in dimension $d=2$. In fact, for $d$ larger
than $2$ the relative loss in precision scales like $N^{(2-d)/(h+1)d}$. Finally, under some conditions we 
observe that there is a value $s^*$ of the penalty coefficient, a free parameter used to fix the number 
of clusters, which separates a fragmentation phase (for $s<s^*$) from a coalescent one (for $s>s^*$)
of the underlying hidden cluster structure. At this precise point holds a self-similarity property which can
be exploited by the hierarchical strategy to actually locate its position. From this observation, 
a strategy based on \AP\ can be defined to find out how many clusters are present in a given dataset.   
}

\section{Introduction}
Since its invention by J. Pearl \cite{Pearl} in the context of Bayesian inference, it has been realized 
that the \emph{belief-propagation} algorithm was
related to many other algorithms  encountered in various fields \cite{KsFrLo} and  it
has since diffused in many different areas (inference problems, signal processing, error codes, image segmentation 
\ldots). In the context of statistical physics, it is closely related to a certain type of mean-field approach (Bethe-Peierls),
more precisely the so-called Bethe-approximation, 
valid on sparse graphs \cite{YeFrWe}. This reconsideration in statistical physics terms, 
has given rise to a new  generation of distributed algorithms. These  
address  NP-hard combinatorial optimization problems, like the \emph{survey-propagation} algorithm 
of MÈzard and Zecchina \cite{MeZe} for the random K-SAT problems, where the factor-graph has a 
tree-like structure. Surprisingly enough, in some other context it works also well on dense factor-graphs
as exemplified by  the \emph{affinity propagation} algorithm proposed by 
Frey and Dueck \cite{frey} for the clustering problem, which is also NP-hard. 
Akin $K$-centers, AP maps each data item onto
an actual data item, called {\em exemplar}, and all items mapped onto the
same exemplar form one cluster. Contrasting with  $K$-centers, AP builds quasi-optimal
clusters in terms of distortion, thus enforcing
the cluster stability \cite{frey}. The price to pay for these  
understandability and stability properties is a  quadratic computational
complexity, except if the similarity matrix is made sparse with help of a pruning procedure. 
Nevertheless, a pre-treatment of the data would also be quadratic in the number if item, 
which is  severely hindering the  usage of AP on large scale datasets. The basic assumption 
behind AP, is that cluster are of spherical shape. This limiting assumption has actually been addressed 
by Leone and co-authors in \cite{LeSuWe1,LeSuWe2}, by softening  a hard constraint present in AP, 
which impose that any  exemplar has first to point to itself as oneself exemplar. 
Another drawback, which is actually common to most clustering techniques, 
is that there is a free parameter to fix which ultimately determines the 
number of clusters. Some methods based on EM \cite{Dempster,Fraley} or on information-theoretic consideration
have been proposed\cite{StBi}, but mainly use a precise parametrization of the cluster model.
There exist also a different strategy based  on similarity statistics~\cite{DuFr},  
that have been already recently combined with AP~\cite{WaZhLi}, at the expense of a quadratic price.
In an earlier work \cite{xlzhang08,xlzhang09}, a hierarchical
approach, based on a divide and conquer strategy was proposed to decrease the AP complexity and  adapt
AP to the context of Data Streaming. In this paper we extend the scaling analysis of 
this procedure initiated in \cite{xlzhang09} and propose a way to determine the number of clusters.

The paper is organized as follows.In  Section~\ref{BPAP} we start
from a brief description of BP and some of its properties. We see how AP can be derived
from it and  present some extensions to \AP, including the  
soft-constraint affinity propagation extension (\SCAP) to \AP. 
In Section~\ref{hap}, the computational complexity of \HAP\  is analyzed
and the leading behavior,  
of the resulting error measured on the distribution of 
exemplars, which depends on the dimension and on the size of the subsets, is computed.
Based on these results we enforce  the self-similarity  of \HAP\ in Section~\ref{rap} 
to develop a renormalized version of 
AP (in the statistical physics sense). We finally discuss how to fix in a self-consistent way
the penalty coefficient present in \AP, which is conjugate to the number of cluster.

\section{Introduction to belief-propagation}\label{BPAP}
\subsection{Local marginal computation}
The belief propagation algorithm is intended to computing
marginals of joint-probability measure of the type
\begin{equation}\label{def:joint}
P({\bf x}) = \prod_a\psi_a(x_a)\prod_i\phi(x_i),
\end{equation}
where ${\bf x} = (x_1,\ldots,x_N)$ is a set of variables,
$x_a = \{x_i,i\in a\}$ a subset of variables involved in the factor
$\psi_a$, while the $\phi_i$'s are single variable factors. 
The structure of the joint measure $P_a$ is conveniently
represented by a factor graph \cite{KsFrLo}, i.e. a bipartite graph with two set of
vertices, $\cal F$ associated to the factors, and $\cal V$ associated
to the variables, and a set of edges $\cal E$ connecting the variables
to their factors (see Figure~\ref{fig:factorgraph}. Computing the single variables marginals scales in 
general exponentially with the size of the system, except when the
underlying factor graph has a tree like structure. In that case all
the single site marginals may be computed at once, by solving the
following iterative scheme due to J. Pearl \cite{Pearl}:
\begin{align*}
m_{a\to i}(x_i) &\longleftarrow \sum_{x_j\atop j\in a,j\ne i
}\psi_a(x_a)\prod_jn_{j\to a}(x_j)\\[0.2cm]
n_{i\to a}(x_i) &\longleftarrow \phi_i(x_i)\prod_{b\ni i,b\ne a} m_{b\to i}(x_i).
\end{align*}
$m_{a\to i}(x_i)$ is called the message sent by factor node $a$
to variable node $i$, while $n_{i\to a}(x_i)$ is the message sent by 
variable node $i$ to $a$. These quantities would actually appear as
intermediate computations terms, while deconditioning (\ref{def:joint}).
On a singly connected factor graph, starting from the leaves, two
sweeps are sufficient to obtain  the fixed points messages, and the 
beliefs (the local marginals) are then obtained from these sets 
of messages using the formulas:
\begin{align*}
b(x_i) &= \frac{1}{Z_i}\phi_i(x_i)\prod_{a\ni i}m_{a\to i}(x_i)\\[0.2cm]
b_a(x_a) &= \frac{1}{Z_a}\psi_a(x_a)\prod_{i\in a}n_{i\to a}(x_i)
\end{align*}
On a multiply connected graph, this scheme can be used as an
approximate procedure to compute the marginals, still reliable
on sparse factor graph, while avoiding the exponential complexity
of an exact procedure. Many connections with mean field approaches of
statistical  physics have been recently unravelled, in particular the 
connection with the TAP equations introduced in the context of spin
glasses \cite{Kaba}, and the Bethe approximation of the free energy which we detail now.
\begin{figure}[ht]
\begin{center}
\resizebox*{.6\textwidth}{!}{\input{factor_graph.pstex_t}}
\end{center}
\caption{Example of a factor-graph representing a joint measure of five variables (circles) and three factors (squares)}
\label{fig:factorgraph}
\end{figure}
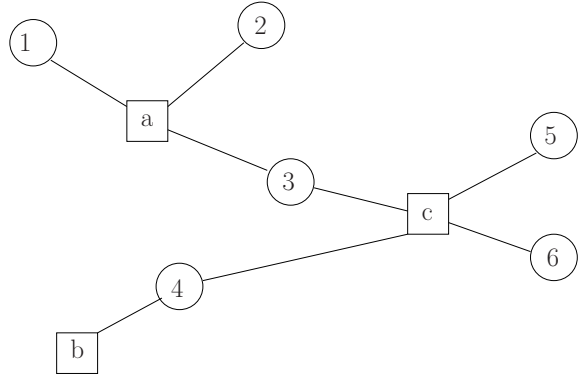

\subsection{\AP\ and \SCAP\ as  min-sum algorithms} \label{AP}
The \AP\ algorithm is a message-passing procedure proposed by Frey and Dueck \cite{frey} that performs 
a classification by identifying exemplars. 
It solves the following optimization problem
\[
{\bf c}^* = \argmin\bigl(E[{\bf c}]\bigr),
\]
with
\begin{equation}\label{def:energyAP}
E[{\bf c}] \egaldef -\sum_{i=1}^N S(i,c_i) - \sum_{\mu=1}^N\log\chi_\mu[{\bf c}]
\end{equation}
where ${\bf c}=(c_1,\ldots,c_N)$ is the mapping between data and exemplars,
$S(i,c_i)$ is the similarity function between $i$ and its exemplar.
For datapoints embedded in an Euclidean space, the common choice for $S$ is the negative squared
Euclidean distance.
A free positive parameter is given by
\[
s \egaldef -S(i,i),\qquad  \forall i,
\]
the penalty for being oneself exemplar.
$\chi_\mu^{(p)}[{\bf c}]$ is a set of constraints.
They read
\[
\chi_\mu[{\bf c}] = \begin{cases}
p, \qquad\text{if}\ c_\mu\ne \mu,\exists i\ s.t.\ c_i=\mu,\\[0.2cm]
1,\qquad\text{otherwise}.
\end{cases}
\]
$p=0$ is the  constraint of the model of Frey-Dueck.
Note that this strong constraint is well adapted to well-balanced clusters, but probably 
not to ring-shape ones.
For this reason 
Leone et. al. \cite{LeSuWe1,LeSuWe2} have introduced the smoothing parameter $p$.
Introducing the inverse temperature  $\beta$, 
\[
P[{\bf c}] \egaldef \frac{1}{Z} \exp(-\beta E[{\bf c}])
\]
represents a probability distribution over clustering assignments $c$.
At finite $\beta$ the classification problem reads
\[
{\bf c}^* = \argmax\bigl(P[{\bf c}]\bigr).
\]
The \AP\ or \SCAP\ equations can be obtained from the standard BP equation \cite{frey,LeSuWe1}
as an instance of the Max-Product algorithm.
For self-containess, let us reproduce the derivation here. 
The  {\sc BP} algorithm provides an approximate 
procedure to the evaluation of the set of single marginal probabilities
$\{P_i(c_i=\mu)\}$ while the min-sum version obtained after taking
$\beta\to\infty$ yields the affinity propagation algorithm of Frey and Dueck.
\begin{figure}[ht]
\begin{center}
\resizebox*{.6\textwidth}{!}{\input{APfactorgraph.pstex_t}}
\end{center}
\caption{Factor graph corresponding to \AP. Small squares represents 
the constraints while large ones are associated to pairwise
contributions in the $E({\bf c})$.}
\label{fig:APfactorgraph}
\end{figure}
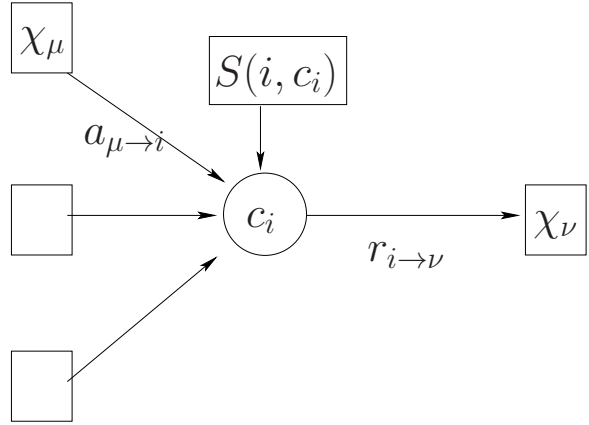
The factor-graph involves variable nodes $\{i,i=1\ldots N\}$ with corresponding variable
$c_i$ and factor nodes $\{\mu,\mu=1\ldots N\}$ corresponding to the energy terms and to
the constraints (see Figure~\ref{fig:APfactorgraph}). Let $A_{\mu\to i}(c_i)$ the message sent by factor $\mu$ to
variable $i$ and $B_{i\to\mu}(c_i)$ the  message sent by variable $i$
to node $\mu$. 
The belief propagation fixed point equations read:
\begin{align}
A_{\mu\to i}(c_i=c) &= \frac{1}{Z_{\mu\to i}} \sum_{\{c_j\}}\prod_{j\ne i} 
B_{j\to\mu}(c_j) \chi_\mu^\beta[\{c_j\},c]\label{APupdate1}\\[0.2cm]
B_{i\to\mu}(c_i=c) &= \frac{1}{Z_{i\to\mu}}\prod_{\nu\ne \mu}A_{\nu\to i}(c)e^{\beta S(i,c)}\label{APupdate2}
\end{align}
Once this scheme has converged, the fixed points messages provide 
a consistency relationship between the two sets of beliefs 
\begin{align}
b_\mu[\{c_i\}={\bf c}] 
&= \frac{1}{Z_\mu} \chi_\mu^\beta[{\bf c}]\prod_{i=1}^N B_{i\to\mu}(c_i)\\[0.2cm]
b_i(c_i = c) &= \frac{1}{Z_i}\prod_{\mu=1}^N A_{\mu\to i}[c]e^{\beta S(i,c)}
\end{align}
The joint probability measure  then rewrites
\[
P[{\bf c}] = \frac{1}{Z_b}\frac{\prod_{\mu=1}^N b_\mu[{\bf c}]}{\prod_{i=1}^N b_i^{N-1}(c_i)}
\]
with $Z_b$ the normalization constant associated to this set of beliefs.
In (\ref{APupdate1}) we observe first that  
\begin{equation}
\hat A_{\mu\to i} \egaldef A_{\mu\to i}(c_i=\nu\ne\mu),\label{eq:msgsymetrie} 
\end{equation}
is independent of $\nu$ and secondly that $A_{\mu\to i}(c_i=c)$  depends only on 
$B_{j\to\mu}(c_j=\mu)$ and on $\sum_{\nu\ne\mu}B_{j\to\mu}(c_j=\nu)$. This means that the 
scheme can be reduced to the propagation of four quantities, by 
letting 
\begin{align*}
A_{\mu\to i}&\egaldef A_{\mu\to i}(c_i=\mu),\\[0.2cm]
\hat A_{\mu\to i} &\egaldef \frac{1- A_{\mu\to i}}{N-1}\\[0.2cm]
B_{i\to\mu} &\egaldef B_{i\to\mu}(c_i=\mu)\\[0.2cm]
\bar B_{i\to\mu} &\egaldef 1- B_{i\to\mu},
\end{align*}
which reduce to two types of messages $A_{\mu\to i}$ and $B_{i\to\mu}$.
The belief propagation equations reduce to
\begin{align*}
A_{\mu\to i} &= \frac{p+(1-p)B_{\mu\to \mu}}{p+(1-p)B_{\mu\to\mu} + 
(N-1)\Bigl[p+(1-p)\bigl(B_{\mu\to\mu}
+\prod_{j\ne i}\bar B_{j\to\mu}\bigr)\Bigr]},\qquad \mu\ne i\\[0.2cm]
A_{i\to i} &= \frac{1}{1 + (N-1)\Bigl[p+(1-p)\prod_{j\ne i}\bar 
B_{j\to i}\Bigr]},\\[0.2cm]
B_{i\to\mu} &= \frac{1}{1+ (N-1)\sum_{\nu\ne\mu} 
\frac{A_{\nu\to i}}{\hat A_{\nu\to i}}e^{\beta (S(i,\nu)-S(i,\mu))}},
\end{align*}
while the approximate single variable belief reads
\[
P_i(c_i=\mu)  = \frac{1}{Z_i}\ \frac{A_{\mu\to i}}{\hat A_{\mu\to i}} 
\ e^{\beta S(i,\mu)}.
\]
This simplification here is actually the key-point in the effectiveness of AP, because 
this let the complexity of this algorithm, which was potentially $N^N$ 
becomes $N^2$, as will be  more obvious later. 
At this point we introduce the log-probability ratios,
\begin{align*}
a_{\mu\to i} &\egaldef \frac{1}{\beta} \log\Bigl(\frac{A_{\mu\to i}}{\hat A_{\mu\to i}}\Bigr),\\[0.2cm]
r_{i\to \mu} &\egaldef \frac{1}{\beta}\log\Bigl(\frac{B_{i\to\mu}}{\bar B_{i\to\mu}}\Bigr), 
\end{align*}
corresponding respectively to the ``availability'' and ``responsibility''
messages of Frey-Dueck.
At finite $\beta$  the equation reads
\begin{align*}
e^{\beta a_{\mu\to i}} &= \frac{e^{-\beta q}+(1-e^{-\beta q})e^{\beta r_{\mu\to\mu}}}
{e^{-\beta q}+e^{\beta r_{\mu\to\mu}} +(1-e^{-\beta q})
\bigl(1+e^{\beta r_{\mu\to\mu}}\bigr)
\prod_{j\ne i}\bigl(1+e^{\beta r_{j\to\mu}}\bigr)^{-1} }\\[0.2cm]
e^{-\beta a_{i\to i}} &= 
e^{-\beta q} +(1-e^{-\beta q})
\prod_{j\ne i}\bigl(1+e^{\beta r_{j\to i}}\bigr)^{-1}\\[0.2cm]
e^{-\beta r_{i\to \mu}} &= \sum_{\nu\ne\mu}
e^{-\beta\bigl(S(i,\mu)-a_{\nu\to i} - S(i,\nu)\bigr)},
\end{align*}
with $q\egaldef -\frac{1}{\beta}\log p$.
Taking the limit $\beta\to\infty$ at fixed $q$ yields
\begin{align}
a_{\mu\to i} &= \min\Bigl(0,\max\bigl(-q,\min(0,r_{\mu\to\mu})\bigr)+\sum_{j\ne i}
\max(0,r_{j\to\mu})\Bigr),\qquad\mu\ne i,\label{eq:AP1}\\[0.2cm]
a_{i\to i} &= \min\Bigl(q,\sum_{j\ne i}
\max(0,r_{j\to i})\Bigr),\label{eq:AP2}\\[0.2cm]
r_{i\to\mu} &= S(i,\mu) - \max_{\nu\ne\mu}
\bigl(a_{\nu\to i} +S(i,\nu)\bigr).\label{eq:AP3}
\end{align}
After reaching a fixed point, exemplars are obtained according to 
\begin{equation}\label{eq:AP4}
c_i^* = \argmax_\mu\bigl(S(i,\mu)+a_{\mu\to i}\bigr) =  
\argmax_\mu\bigl(r_{i\to\mu} + a_{\mu\to i}\bigr).
\end{equation}
Altogether, \ref{eq:AP1},\ref{eq:AP2},\ref{eq:AP3} and \ref{eq:AP4} constitute the equations of \SCAP\ 
which reduce to the equations of \AP\ when $q$ tends to $-\infty$.

\section{Hierarchical affinity propagation (\HAP)}\label{hap}
\subsection{Weighted affinity propagation (\WAP)}\label{wap}
Assume that a subset ${\cal S}\subset {\cal E}$ 
of $n$ points, assumed to be at a small average mutual distance $\epsilon$
are aggregated into a single point $c\in {\cal S}$. The similarity matrix has to be modified as follows
\begin{align*}
S(c,i) &\longrightarrow nS(c,j),\qquad\qquad\forall i \in\bar{\cal S}\\[0.2cm]
S(i,c) &\longrightarrow S(i,c), \qquad\qquad\ \ \forall i \in\bar{\cal S}\\[0.2cm]
S(c,c) &\longrightarrow \sum_{i\in {\cal S}} S(i,c),
\end{align*}
and all lines and columns with index $i\in{\cal S}\backslash \{c\}$ are suppressed
from  the similarity matrix.
This type of rules should be applied when performing hierarchies.

\begin{figure}[ht]
\centering
\includegraphics*[width=0.5\columnwidth]{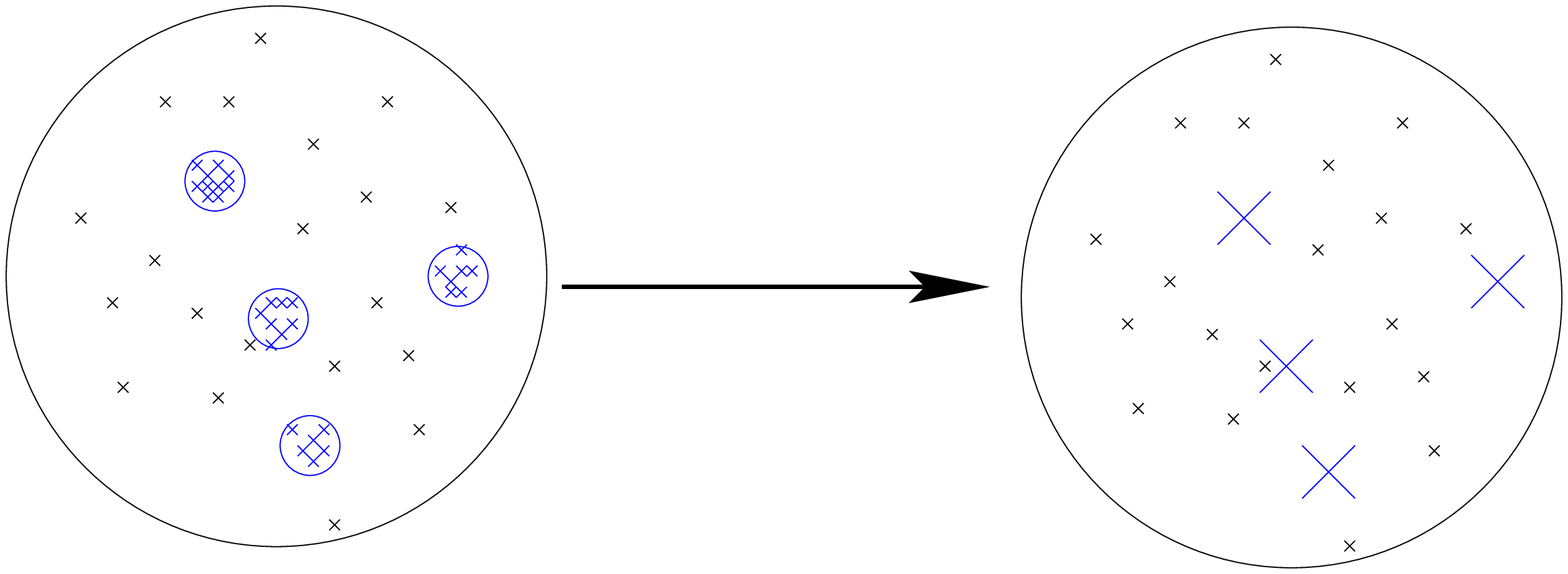}
\caption{\label{fig:wap}
}
\end{figure}
This redefinition of the self-similarity yields a non-uniform penalty coefficient.
In the basic update equations (\ref{eq:AP1}), (\ref{eq:AP2}), (\ref{eq:AP3}) and (\ref{eq:AP4}), nothing prevents
from having different self-similarities because while the key property (\ref{eq:msgsymetrie}) 
for deriving these equations is not affected by this.
When comparing in (\ref{def:energyAP}), the relative contribution of the similarities between 
different points on one hand, and the penalties on the other hand, 
we immediately see that $s$ has to scale like the size of the dataset. 
This  insures a basic scale invariance of the result, i.e. that the same solution is recovered, when the number of 
points in the dataset is rescaled  by some arbitrary factor. 
Now, if we deal directly with weighted data points in an Euclidean space, the preceding considerations 
suggests that  one may start directly from the following rescaled 
cost function:
\begin{equation}\label{eq:energy}
E[{\bf c}] \egaldef \frac{1}{Z}\sum_{c=1}^{n}\sum_{i\in c}\bigl(w_id^2(i,c) + \frac{n}{V} s\bigr). 
\end{equation}
$Z$ is a normalization constant
\[
Z \egaldef \sum_{i\in {\cal S}} w_i,
\]
The similarity measure has been specified with help of the Euclidean distance 
\[
d(i,j) = \vert \r_i -\r_j\vert,\qquad\qquad\forall (i,j)\in{\cal S}^2.
\]
The $\{w_i,\forall i\in{\cal S}\}$ is a set of weights attached to each datapoint and 
the self-similarity has been rescaled uniformly
\[
s\longrightarrow \frac{1}{V}\sum_{i\in{\cal S}} w_i\ s.
\]
with respect to density of the dataset, $V$ being the volume of the embedding space, 
for later purpose (thermodynamic limit in Section~\ref{rap}).

\subsection{Complexity of \HAP}

AP computational complexity\footnote{Except if the similarity matrix
is sparse, in which case the complexity reduces to $Nklog(N)$ with
$k$ the average connectivity of the similarity matrix \cite{frey}.}
 is expected to scale like ${\cal O}(N^2)$; it involves the matrix $S$ of pair distances, with
quadratic complexity in the number $N$ of items, severely hindering
its use on large-scale datasets.

This AP limitation can be overcome through a \DAC\ heuristics
inspired from \cite{guha_TKDE}. Dataset $\cal E$ is randomly split
into  $b$  data subsets as shown on Figure~\ref{fig:hap}; AP is launched on every subset and outputs
a set of exemplars; the exemplar weight is set to the number of
initial samples it represents; finally, all weighted exemplars are
gathered and clustered using WAP\ (the complexity is ${\cal
O}(N^{3/2})$ \cite{xlzhang08}).  This \DAC\ strategy $-$ which could
actually be combined with any other basic clustering algorithm $-$
can be pursued hierarchically in a self-similar way, as a branching
process with $b$ representing the branching coefficient of the
procedure, defining the Hierarchical AP (\HAP) algorithm.

\begin{figure}[ht]
\begin{center}
\resizebox*{.8\textwidth}{!}{\input{hap.pstex_t}}
\end{center}
\caption{\label{fig:hap}
}
\end{figure}
Formally, let us define a tree of clustering operations, where
the number $h$ of successive random partitions of the
data represents the height of the tree. At each level of the
hierarchy, the penalty parameter $\sp$ is set such that the
expected number of exemplars extracted along each clustering step is
upper bounded by some constant $K$.

\begin{prop} Letting the branching factor $b$ to
\[
b = \bigl(\frac{N}{K}\bigr)^\frac{1}{h+1},
\]
then the overall complexity $C(h)$ of \HAP\  is given by
\[
C(h) \propto K^{\frac{h}{h+1}}N^\frac{h+2}{h+1}\qquad N\gg K.
\]
\end{prop}
\begin{proof}
$M = N/b^h$ is the size of each subset to be clustered at level $h$;
at level $h-1$, each clustering problem thus involves $b K = M$ exemplars
with corresponding complexity
\[
C(0) = K^2\bigl(\frac{N}{K}\bigr)^{\frac{2}{h+1}}.
\]
The total number $N_{cp}$ of clustering procedures involved is
\[
N_{cp} = \sum_{i=0}^h b^{i} = \frac{b^{h+1}-1}{b-1},
\]
with overall computational complexity:
\[
C(h) = K^2\bigl(\frac{N}{K}\bigr)^\frac{2}{h+1}
\frac{\frac{N}{K}-1}{\bigl(\frac{N}{K}\bigr)^\frac{1}{h+1}-1}
\equ_{N\gg K} K^2\bigl(\frac{N}{K}\bigr)^\frac{h+2}{h+1}.
\]
It is seen that $C(0) = N^2$, $C(1)\propto N^{3/2}$,\ldots, and
$C(h)\propto N$ for ${h\gg 1}$ .
\end{proof}

\subsection{Information loss of \HAP}\label{sec:proof}
Let us examine the price to pay for this complexity reduction. As
mentioned earlier on, the clustering quality is usually assessed from its
distortion, the sum of the squared distance between every data item
and its exemplar:
\[ D(\c) = \sum_{i=1}^N d^2(e_i,c_i) \]

\begin{figure}[ht]
\centering
\includegraphics*[width=0.45\columnwidth]{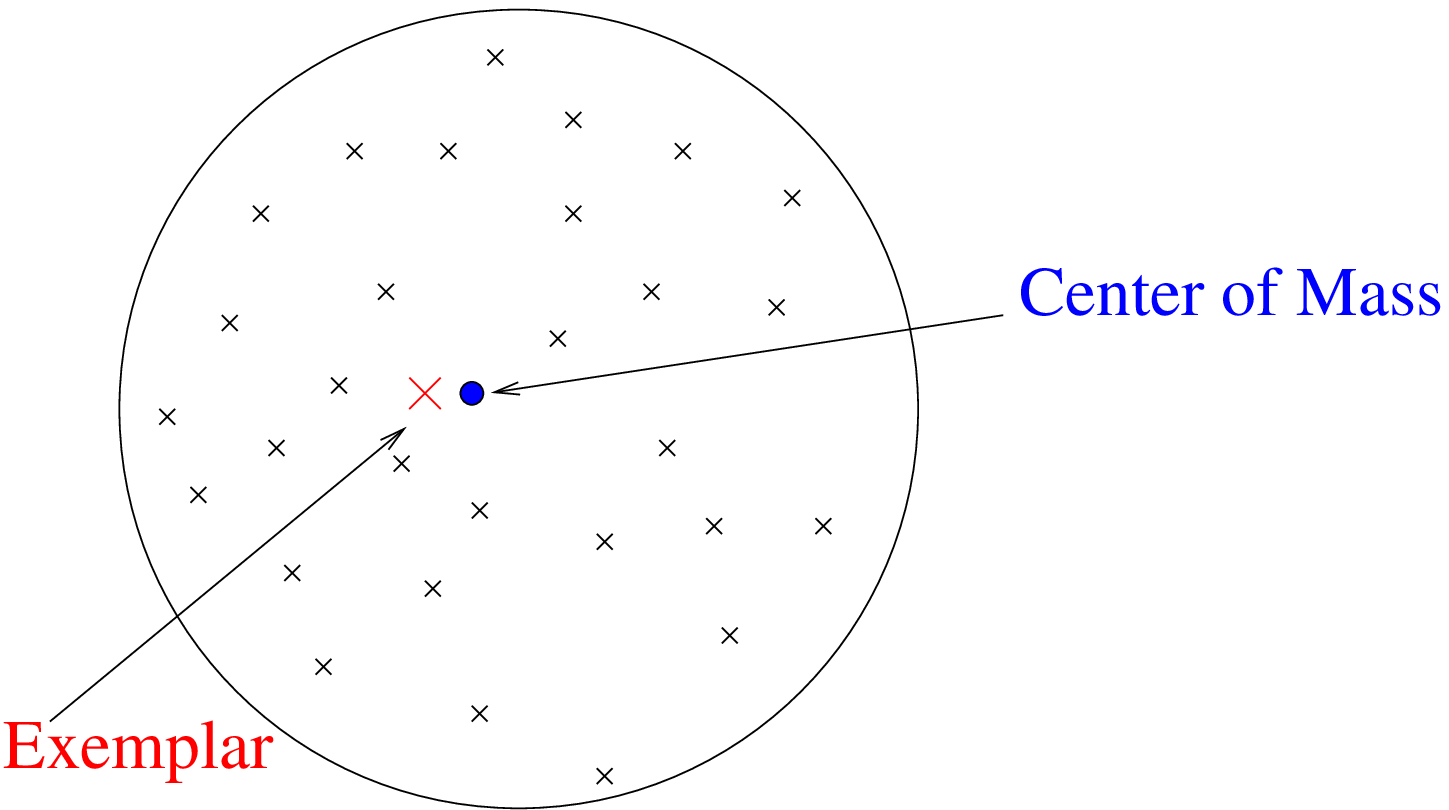}
\caption{\label{fig:singlecluster}}
\end{figure}

The information loss incurred by \HAP\ w.r.t. AP
is examined in the simple case where the data samples follow
a centered distribution in
$\RR^d$. By construction, AP aims at finding the cluster exemplar  \rc\
nearest to the center of mass of the sample points noted \rcm:
\[
D(\c) = \vert \rcm\ - \rc \vert^2 + Cst
\]
To assess the information loss incurred by \HAP\ it turns out to be more convenient
to compare the results in distribution. This can be done by considering e.g. the
relative entropy, or Kullback Leibler distance, between the distribution \Pc\
of the cluster exemplar computed by AP, and the distribution \Pch\ of the cluster
exemplar computed by \HAP\ with hierarchy-depth $h$:
\begin{equation}\label{DKL}
D_{KL}\bigl(\Pc||\Pch\bigr) = \int \Pch(r) \log\frac{\Pch(r)}{\Pc(r)}
d{\bf r}
\end{equation}

In the simple case where points are sampled along a
centered distribution in $\RR^d$, let \rt\ denote the relative
position of exemplar \rc\ with respect to the center of mass \rcm:
$$\rt = \rc - \rcm$$
The probability distribution of \rt\ conditionally to \rcm\ is
cylindrical; the cylinder axis supports the segment (0, \rcm), where $0$ is the origin of the $d$-dimensional
space. As a result, the probability distribution of \rcm + \rt\ is the convolution of a
spherical with a cylindrical distribution.

Let us introduce some notations. Subscripts  $sd$ refers to sample data, $ex$ to the
exemplar, and $cm$ to center of mass.
Let $x_{\centerdot}$ denote the corresponding square distances to the origin,
 $f_{\centerdot}$ the corresponding probability densities and $F_{\centerdot}$ their
cumulative distribution. Assuming
\begin{equation}\label{def:sigma}
\sigma \egaldef {\mathbb E}[x_{sd}] = \int_0^\infty x f_{sd}(x)dx ,
\end{equation}
and
\begin{equation}\label{def:alpha}
\alpha \egaldef - \lim_{x\to 0}\
\frac{\log(F_{sd}(x))}{x^\frac{d}{2}},
\end{equation}
exist and are finite, then the cumulative distribution  of $x_{cm}$
of a sample of size $M$ satisfies
\[
\lim_{M\to\infty} F_{cm}(\frac{x}{M}) =
\frac{\Gamma\bigl(\frac{d}{2},\frac{dx}{2\sigma}\bigr)}{\Gamma\bigl(\frac{d}{2}\bigr)}.
\]
by virtue of the central limit theorem. In the meanwhile, \tx = $\vert\r_{ex} - \r_{cm}\vert^2$ has a
universal  extreme value distribution (up to rescaling, see e.g. \cite{EXVAL1} for general methods):
\begin{equation}\label{Weibull}
\lim_{M\to\infty} F_{\widetilde{ex}}(\frac{1}{M^{2/d}}x) = \exp\bigl(-\alpha
x^\frac{d}{2}\bigr).
\end{equation}
To see how the clustering error propagates along with the
hierarchical process, one proceeds inductively. At hierarchical level $h$,
$M$ samples,  spherically distributed with variance $\sigma^{(h)}$ are considered;
the sample nearest to the center of mass is selected as exemplar. Accordingly, at
hierarchical level $h+1$, the next sample data is distributed after the
convolution of two spherical distributions, the exemplar and center of mass
distributions at level $h$. The following scaling
recurrence property (proof in  appendix) holds:
\begin{prop}\label{scaling}
\[
\lim_{M\to\infty} F_{sd}^{(h+1)}(\frac{x}{M^{(h+1)\gamma}})\ =\ 
\begin{cases}
\DD
\frac{\Gamma(\frac{1}{2},\frac{x}{2\sigma^{(h+1)}})}{\Gamma(\frac{1}{2})}
\qquad\qquad    d=1 \ ,\   \gamma = 1\\[0.3cm]
\DD \exp(-\alpha_2^{(h+1)}x) \qquad    \   d=2  \ , \gamma = 1\\[0.3cm]
\DD \exp\bigl(-\alpha^{(h+1)}x^{\frac{d}{2}}\bigr)
\qquad  d>2 \ ,  \gamma = \frac{2}{d}
\end{cases}
\]
with
\[
\sigma^{(h+1)} = \sigma^{(h)}, \qquad \alpha^{(h+1)}
= \alpha^{(h)}, \quad \alpha_2^{(h+1)}= \frac{\alpha_2^{(h)}}{2} .
\]
It follows that the distortion loss incurred by \HAP\ does not depend on the hierarchy 
depth $h$ except in dimension $d=2$.
\end{prop}
\begin{proof}
See appendix~\ref{proofprop}.
\end{proof}
\begin{figure}[ht]
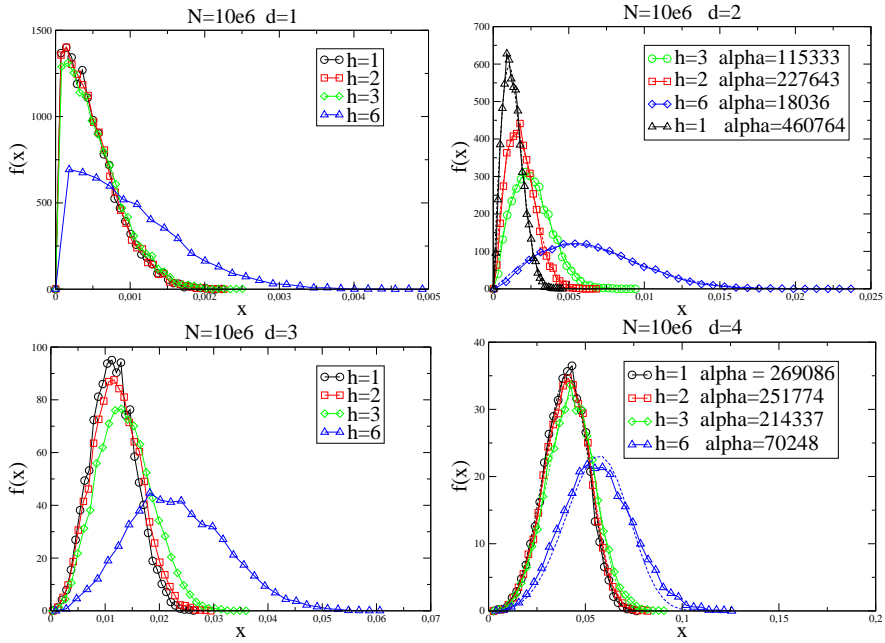

\centering
\includegraphics*[width=0.45\columnwidth]{10e6_1_h.eps}
\includegraphics*[width=0.45\columnwidth]{10e6_2_h.eps}
\includegraphics*[width=0.45\columnwidth]{10e6_3_h.eps}
\includegraphics*[width=0.45\columnwidth]{10e6_4_h.eps}
\caption{\label{distrib} Radial distribution plot of exemplars
obtained by clustering of Gaussian distributions of $N=10^6$ samples in $\RR^d$
in one single cluster exemplar, with hierarchical level $h$ ranging in 1,2,3,6, for diverse
values of $d$: $d=1$ (upper left), $d=2$ (upper right), $d=3$
(bottom left) and $d=4$ (bottom right). Fitting  functions are of
the form $f(x) = C x^{d/2-1}\exp(-\alpha x^{d/2})$.}
\end{figure}
\vspace{-.1in}

Figure~\ref{distrib} shows the radial distribution of exemplars obtained with different
hierarchy-depth $h$ and depending on the dimension $d$ of the dataset. The  curve for $h=1$ corresponds
to the AP case so the comparison with $h>1$ shows that the information loss due 
to the hierarchical approach is
moderate to negligible in dimension $d\ne 2$ provided that the  number of samples per cluster at
each clustering level is ``sufficient'' (say, $M>30$ for the law of
large numbers to hold).
In dimension $d>2$, the distance of the center of mass to the origin is negligible with
respect to its distance to the nearest exemplar; the distortion behaviour thus is
given by the Weibull distribution which is stable by definition (with an increased
sensitivity to small sample size $M$ as $d$ approaches  2). In dimension $d=1$, the
distribution is dominated by the variance of the center of mass,
yielding the gamma law which is also stable with respect to the
hierarchical procedure. In dimension $d=2$ however, the Weibull and
gamma laws do mix at the same scale; the overall effect is that the width of the
distribution  increases like $2^h$, as shown in Fig.~\ref{distrib} (top right).

We can also compute the corrections to this when $M$ is finite. We have the following property 
which is valid for (non necessarily spherical) distributions of sample points, with a finite variance $\sigma$.
\begin{prop}\label{prop:corrections}
For $d>2$, at level $h$, assume that
\begin{equation}\label{def:alpha2}
\alpha^{(h)} \egaldef p_{sd}^{(h)}(0)\frac{\Omega_d}{d},
\end{equation}
with $\Omega_d = 2\pi^{d/2}/\Gamma(d/2)$ the $d$-dimensional solid angle, and
\[
\sigma^{(h)} = \int d^d\r r^2 p_{sd}^{(h)}(\r)
\]
are both finite. Defining the shape factor of the distribution by
\begin{equation}\label{def:formfactor}
\omega^{(h)} \egaldef \frac{\sigma^{(h)}{\alpha^{(h)}}^{2/d}}{\Gamma\bigl(1+\frac{2}{d}\bigr)},
\end{equation}
the recurrence then reads,
\begin{equation}\label{eq:rec}
\begin{cases}
\DD \sigma^{(h+1)} = \sigma^{(h)}\Bigl(1+\frac{\omega^{(h)}-\omega^{(h-1)}}{M^{1-2/d}}\Bigr)
+o\bigl(M^{2/d-1}\bigr).\\[0.2cm]
\DD \omega^{(h+1)} = 1 + \frac{\omega^{(h)}}{M^{1-2/d}} +o\bigl(M^{2/d-1}\bigr).
\end{cases}
\end{equation}
for $h>0$ and 
\[
\sigma^{(1)} = \frac{\sigma^{(0)}}{\omega^{(0)}}\Bigl(1+\frac{\omega^{(0)}}{M^{1-2/d}}\Bigr)
+o\bigl(M^{2/d-1}\bigr),
\]
\end{prop}
\begin{proof}
See appendix~\ref{corrections}. 
Note that the definition $\alpha^{(h)}$ is equivalent to the previous definition (\ref{def:alpha})
if the distribution of sample points is regular
(as well as for the variance $\sigma^{(h)}$ obviously). 
\end{proof}
As a result, for  $h=1$ we obtain
\[
\sigma^{(2)} = \sigma^{(1)}\Bigl(1+\frac{1-\omega^{(0)}}{M^{1-2/d}}\Bigr)
+o\bigl(M^{2/d-1}\bigr),
\]
and thereby for $h>1$,
\begin{align*}
\sigma^{(h+1)} &= \sigma^{(h)}
+o\bigl(M^{2/d-1}\bigr),\\[0.2cm]
&= \frac{\sigma^{(0)}}{\omega^{(0)}}\Bigl(1+\frac{1}{M^{1-2/d}}\Bigr)
+o\bigl(M^{2/d-1}\bigr)
\end{align*}
This means that the mean error within the 
hierarchical procedure compared to the expected error scales like $N^{\frac{2}{hd}-\frac{1}{h}}$. 
\begin{figure}[ht]
\centering
\includegraphics*[width=0.7\columnwidth]{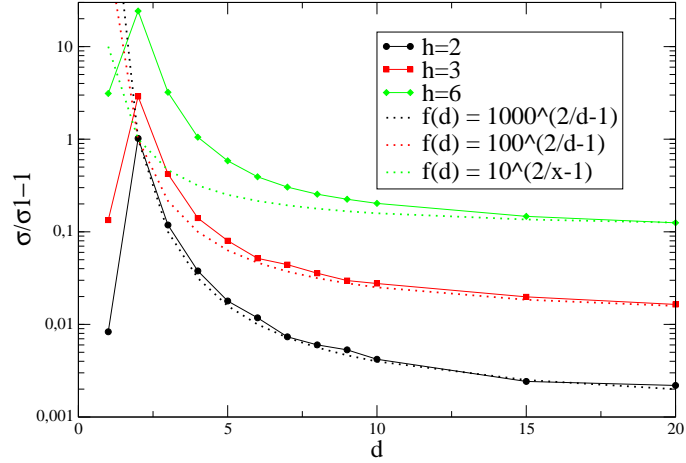}
\caption{\label{fig:sigma_d} $\sigma_{ex}^{(h)}/\sigma_{ex}^{(1)}-1$ for $h=2,3,6$ as a function of the dimension,
when finding exemplars of a single cluster of $10^6$ points (repeated $10^4$ times)}
\end{figure}
From  definition (\ref{def:formfactor}) and (\ref{def:alpha2}), 
the relation between the shape factor $\omega$ and the variance $\sigma$  and the density $p_0$ near the 
center of mass reads,
\[
\omega \egaldef \frac{\bigl(\frac{d}{2}\bigr)^{1-2/d}\pi}
{\Ga\bigl(\Gamma\bigl(\frac{d}{2}\bigr)\bigr)^{2/d}}\ p_0^{2/d}\sigma
\eq_{d\to\infty} \frac{2}{d} p_0^{2/d}\sigma.
\]
As its name indicates, it depends on the shape
of the clusters. It relates the variance of the cluster to its density in the vicinity of the center of mass.
In what follows it will be useful to keep in mind the following
particular distributions  in dimension $d$:
\[
\omega = 
\begin{cases}
\DD \frac{d}{2}\frac{\pi}{\Gb},\qquad\text{Gaussian distribution}\\[0.4cm]
\DD \frac{d}{d+2}\frac{1}{\Gb},\qquad\text{uniform $L2$-sphere distribution}\\[0.4cm]  
\DD \frac{\pi}{6}
\frac{\bigl(\frac{d}{2}\bigr)^{2-2/d}}{\Ga\bigl(\Gamma\bigl(\frac{d}{2}\bigr)\bigr)^{2/d}},
\qquad\text{uniform $L1$-sphere distribution},  
\end{cases}
\]
while  by our definition it is equal to unity for the Weibull distribution~\ref{Weibull}
yielded by the clustering procedure.
In addition, for a superposition of clusters with identical forms and weights, represented by a distribution 
of the type (\ref{def:mixture}), by simple inspection,
the shape factor of the mixture is strictly greater than the single cluster shape factor, a soon as
two cluster do not coincide exactly.

\section{Renormalized  affinity propagation (\RAP)}\label{rap}
The self-consistency of the \HAP\ procedure 
may be exploited to determine the underlying number of clusters present in  
a given data set. We follow here the guideline of 
the  standard renormalization (or decimation) procedure which is used in 
statistical physics for analyzing scaling properties.\\[0.2cm]
\begin{figure}[ht]
\begin{center}
\resizebox*{.6\textwidth}{!}{\input{renorm.pstex_t}}
\end{center}
\caption{}\label{fig:renorm}
\end{figure}
{\bf basic scaling}\\
Consider first a dataset, composed of $N$ items, occupying  a region of total  volume $V$, 
of a $d$-dimensional space, distributed according to 
a superposition of localized distributions,
\begin{equation}\label{def:mixture}
f(\r) = \frac{1}{n^*}\sum_{c=1}^{n^*} \ {\sigma_c^*}^{-\frac{d}{2}}f_0\bigl(\frac{\vert\r-\r_c\vert}{\sqrt{\sigma_c^*}}\bigr), 
\end{equation}
where $n^*$ is the actual number of clusters, $f_0$ is a distribution normalized to one, 
$\r_c$ is the center of cluster $c$ and $\sigma_c^*$ the variance. 
Assume the data set is partitioned into $n = xV$
clusters, $x$ representing the density of these clusters, each cluster containing $N_c$ points and 
occupying an effective volume $V_c$.
The energy  (\ref{eq:energy}) of the clustering 
reads for large $N$ and $n<<N$
\begin{align}
E[{\bf c}] &= \frac{1}{N}\sum_{c=1}^{xV}\Bigl[\sum_{i\in c} d^2(i,c) + \frac{N}{V} s\Bigr]\nonumber\\[0.2cm]
&= \sigma(x) +x s,\label{eq:thlim}
\end{align}
where 
\begin{align}
\sigma(x) &\egaldef  \frac{1}{N}\sum_{c=1}^{xV}\sum_{i\in c} d^2(i,c),\nonumber\\[0.2cm]
&= \sum_{c=1}^{xV}\nu_c \sigma_c,\label{eq:distor}
\end{align}
is the distortion function, with $\nu_c = N_c/N$ is the fraction of data-point and $\sigma_c$ the corresponding 
variance of the \AP-cluster $c$:
\[
\sigma_c \egaldef \int d^d\r f_c^{(AP)}(\r)(\r-\r_c)^2
\]
by virtue of the law of large numbers. If $V_c$ represents the effective volume of such  a cluster, 
we have 
\[
\sigma_c = V_c^{2/d} \tilde\sigma_c,
\]
where $\tilde \sigma_c$ is a dimensionless quantity.
For $n\gg n^*$ 
we expect $V_c = V/n$ and  $\tilde \sigma_c = \sigma$ (all \AP-cluster have same spherical shape in this limit)
so that  
\[
E[{\bf c}] \eq_{x\gg x^*} x^{-2/d}\sigma + xs.
\]
For a given value of $s$, the optimal clustering is obtained for 
\[
x(s) = \Bigl(\frac{2\sigma}{ds}\Bigr)^\frac{d}{2}\qquad\qquad s\ll s^*. 
\]
{\bf self-consistency}\\
Consider now a one step \HAP\ (see Figure~\ref{fig:renorm}) where the $N$-size data set is randomly partitioned
into $M = 1/\lambda$  subsets of  $\lambda N$ points each and where the reduced penalty 
$s$ is fixed to some value
such each clustering procedure yields $n$ exemplars in average. 
The resulting set of data points is then of size $n/\lambda$ and the question is how to 
adjust the value $s^{(\lambda)}$ for the clustering of this  
new data set, in order to recover the same result as the one which is obtained with a direct procedure with penalty $s$. 
\begin{figure}[ht]
\begin{center}
\resizebox*{.8\textwidth}{!}{\input{cluster1.pstex_t}}
\end{center}
\caption{}\label{fig:cluster}
\end{figure}
\begin{figure}[ht]
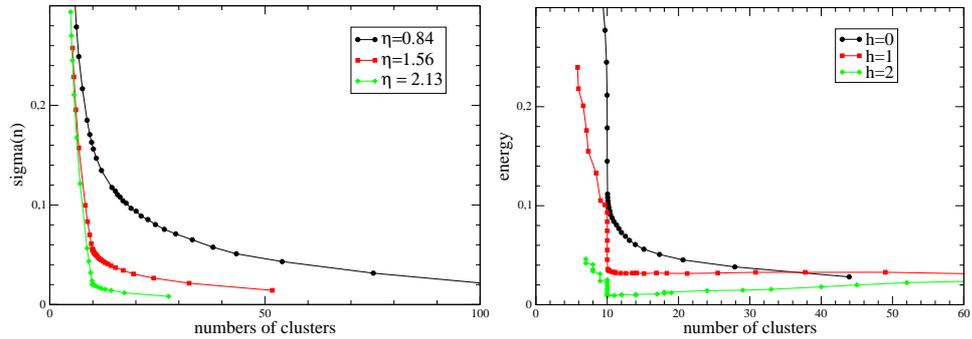

\centering
\includegraphics*[width=0.5\columnwidth]{sigdn_10_5_1.eps}
\includegraphics*[width=0.49\columnwidth]{edn_300_2_10_5_0_0.85.eps}
\caption{\label{sigmadn} The distortion function for various values of $\eta$ for $d=5$
and $n^*=10,30$ (left panel).The energy (i.e.the distortion plus the penalties) 
as a function of the number of clusters obtained at each hierarchical level of a single \HAP\ procedure
for $d=5$, $n^*=10$, $\eta=0.85$, $h=2$ and $\lambda N=300$ (right panel).}
\end{figure}
To answer to this question we make some hypothesis on the data set.
\begin{itemize}
\item ({\bf H1}): the distribution of data points in the original 
set consists in a superposition of $n^*$ non-overlapping distributions,
with common \emph{shape factor $\omega$ (\ref{def:formfactor})}  
\item ({\bf H2}): there exists a value $s^*$ of $s$ for which \AP\ 
yields the correct number of clusters $C$ when  $N$ tends to
infinity.
\item ({\bf H3}):  
$\sigma(x)$ which represents the mean square distance of the sample data to their exemplars in the
thermodynamic limit,
is assumed to be a smooth decreasing convex 
function of the density  $x = n/V$ of exemplars  (obtained by \AP) 
with possibly a cusp at $x=x^*$ (see Figure~\ref{sigmadn}).
\end{itemize}
The first hypothesis essentially amounts to assume that the clusters are ``sufficiently"
separated with respect to their size distribution. This can be characterized by the following parameter
\begin{equation}\label{def:eta}
\eta \egaldef \frac{d_{min}}{2R_{max}},
\end{equation}
where $d_{min}$ is the minimal mutual distance among clusters (between centers) and $R_{max}$ is the 
maximal value of cluster radius. We expect a good separability property for $\eta >1$.   
In practice this means the following. When $s$ is slowly  increased,
so that the number of cluster decreases unit by unit, the disappearing of a cluster corresponds
either to some ``true'' cluster to be partitioned in on unit less, or either to two ``true'' clusters
that are merged into a single cluster. The assumption ({\bf H2}) amounts to say that
the cost in distortion caused by merging two different true clusters  
is always greater than the cost corresponding to the fragmentation of  one of the true clusters. 
As a result, starting from small values of $s$, 
and increasing it slowly, one is witnessing in the first part of the process a  decrease
in the fragmentation  of the true clusters until some threshold value $s^*$ is reached. At that point 
this de-fragmentation process ends  
and is replaced by the merging process of the true clusters (see Figure~\ref{fig:cluster}). In the 
thermodynamic limit (which sustain all the present considerations), $s^*$ can be viewed as a critical
value corresponding to  a (presumably) second order phase transition, which separates 
a coalescent phase from a fragmentation phase.      
Note that when the second hypothesis ({\bf H2}) is satisfied, this implied
that ({\bf H1}) is automatically satisfied by the dataset obtained after the first step of
\HAP\ performed at $s^*$. Indeed, each partition of the initial set yields exactly one exemplar per true
cluster in that case, and the rescaled distribution of these exemplars w.r.t their ``true''
center  are universal after rescaling. 
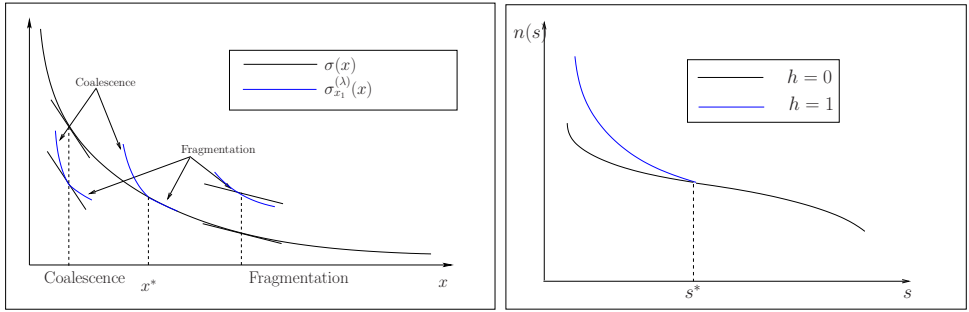
\begin{figure}[ht]
\begin{center}
\resizebox*{.51\textwidth}{!}{\input{sigmadx.pstex_t}}
\resizebox*{.48\textwidth}{!}{\input{ndx.pstex_t}}
\end{center}
\caption{\label{fig:sigmadx}Sketch of the rescaling property. Comparison of 
the distortion function between two stages of \HAP\ (left panel). Corresponding
result in terms of the number of clusters as a function of $s$.} 
\end{figure} 
The renormalization setting is depicted in Figure~\ref{fig:renorm}. 
The dataset is partitioned into $M=1/\lambda$ 
subsets. We manage that the size  $\lambda N$
of each subset  remain the same at each stage of the procedure. This means that $\lambda$ is set 
to $\lambda^2 n$, if $n$ is the expected number of exemplars obtained at the corresponding stage.
Under the two assumptions ({\bf H1}) and ({\bf H2}), the second step of \HAP, 
which corresponds to the clustering of $Nn/\lambda$ data points with a 
penalty set to $Nns(\lambda)/\lambda$ is expected to have the following property.
\begin{prop}\label{prop:renorm}
Let $n_1$ [resp. $n_2$] the number of clusters obtained after the first [resp. second]
clustering step. 
If 
\begin{equation}\label{penascale}
s^{(\lambda)} = \frac{1}{\omega}\bigl(\frac{N\lambda}{n_1}\bigr)^{-\frac{2}{d}}\ s = \frac{\lambda^{2/d}}{\omega}\ s,
\qquad\text{with}\qquad \frac{\lambda^{2/d}}{\omega}\ll 1,
\end{equation}
then either cases occurs,
\[
\begin{cases}
s<s^*\qquad\text{then}\qquad n_2 \ge n_1\ge n^*\\[0.2cm]
s = s^*\qquad\text{then}\qquad n_2 = n_1 = n^*.\\[0.2cm]
s>s^*\qquad\text{then}\qquad n_2 = n_1\le  n^* 
\end{cases}
\]
\end{prop}
\begin{proof}

In the thermodynamic limit the value $n_1$ for $n$, which minimizes the energy 
is obtained for $x_1=n_1/V$ as the minimum of (\ref{eq:thlim}): 
\[
s + \sigma'(x_1) = 0.
\]
At the second stage one has to minimize with respect to $x$,  
\[
E^{(\lambda)}[{\bf c}] = \frac{\lambda^{2/d}}{\omega}\Bigl[\frac{\omega}{\lambda^{2/d}}
\sigma^{(\lambda)}(x,x_1) + xs\Bigr],
\]
where $\sigma^{(\lambda)}(x,y)$ denotes the distortion function
of the second clustering stage when the first one yields a density $y$ of clusters.
This amounts to find $x=x_2$ such that
\begin{equation}\label{eq:hap2}
s + \frac{\omega}{\lambda^{2/d}}\frac{\partial\sigma^{(\lambda)}}{\partial x}(x,x_1) = 0,
\end{equation}
We need now to see  how, depending on $x$,  
\[
\tilde\sigma_x^{(\lambda)}(y) \egaldef \lambda^{-2/d}\sigma^{(\lambda)}(y,x)
\] 
compares with $\sigma(y)$. This is depicted on Figure~\ref{fig:sigmadx}.
\end{proof} 
\begin{figure}[ht]
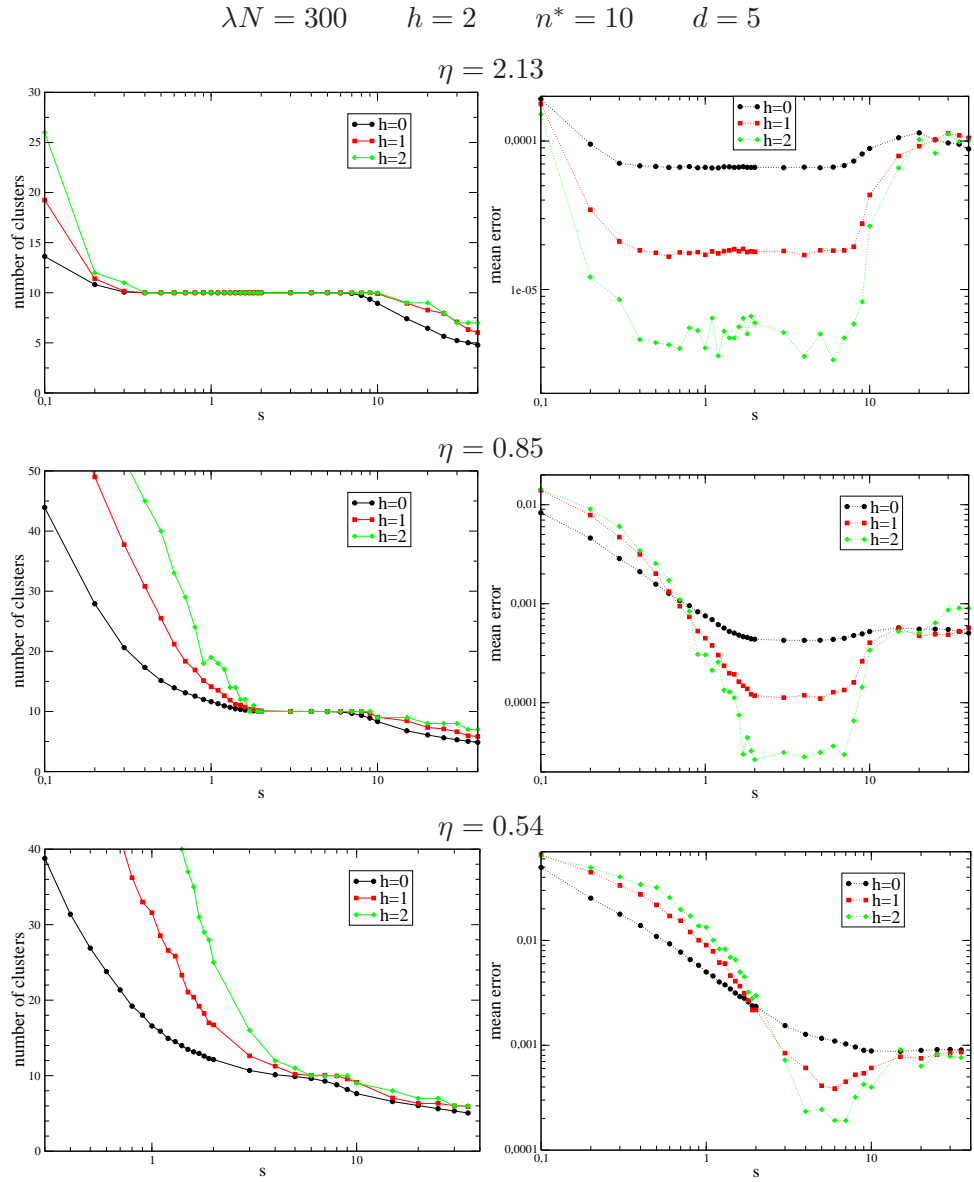

\centering
$\lambda N=300\qquad h=2\qquad n^*=10\qquad d=5$\\[0.2cm]
$\eta=2.13$\\
\includegraphics*[width=0.49\columnwidth]{hap_300_2_10_5_0_2.14.eps}
\includegraphics*[width=0.5\columnwidth]{dist_300_2_10_5_0_2.14.eps}
$\eta=0.85$\\
\includegraphics*[width=0.49\columnwidth]{hap_300_2_10_5_0_0.85.eps}
\includegraphics*[width=0.5\columnwidth]{dist_300_2_10_5_0_0.85.eps}
$\eta=0.54$\\
\includegraphics*[width=0.49\columnwidth]{hap_300_2_10_5_0_0.54.eps}
\includegraphics*[width=0.5\columnwidth]{dist_300_2_10_5_0_0.54.eps}
\caption{\label{renormplots} Number of cluster obtained in one single run with respect to the hierarchical 
level (left panel). Error distance of the exemplars from the true underlying centers obtained with respect
to the hierarchical level (right panel). In all cases, $10$ underlying $L2$-sphere shaped clusters are present}
\end{figure}
\begin{figure}[ht]
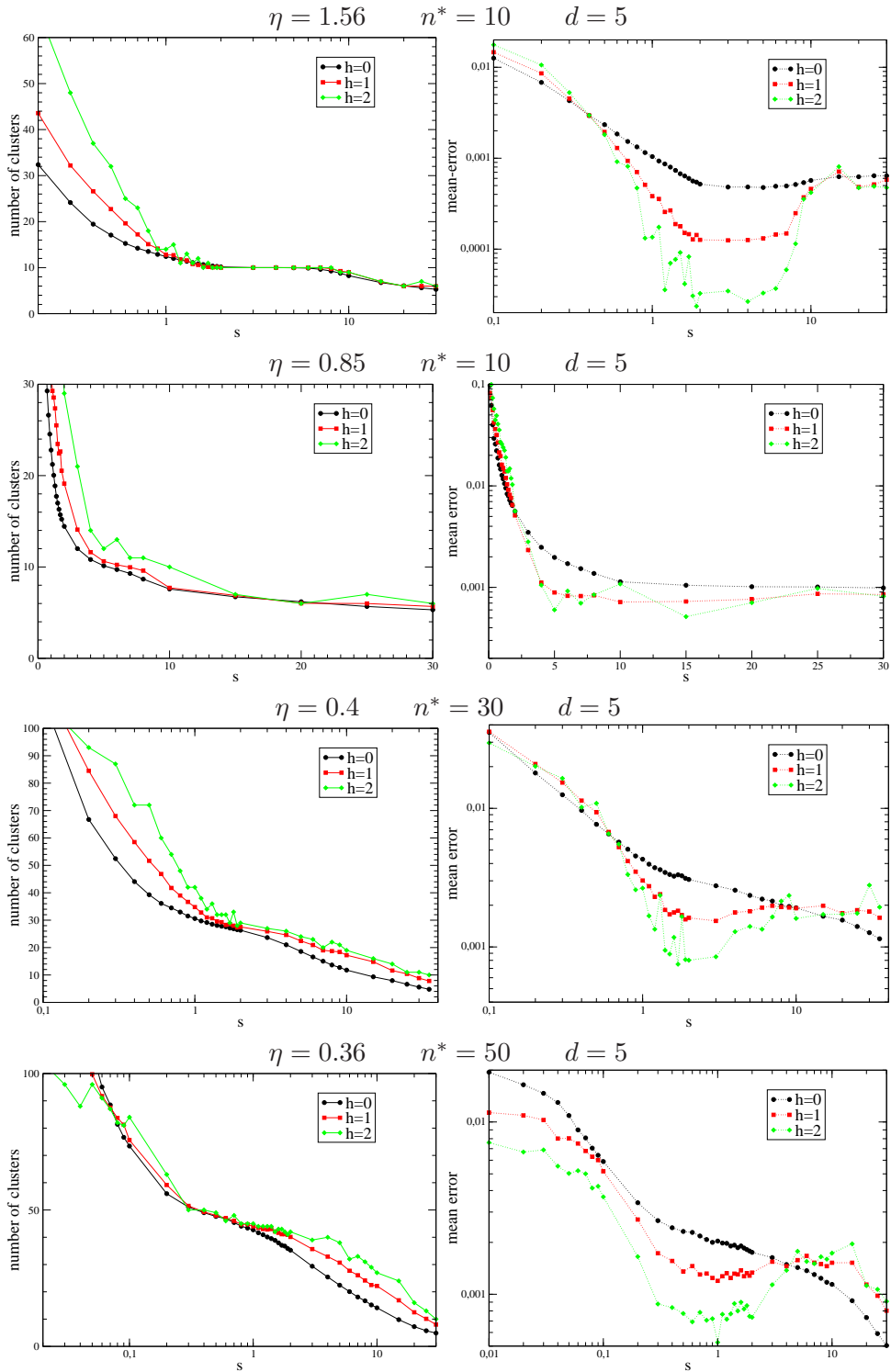

\centering
$\eta=1.56\qquad n^*=10\qquad d=5$\\
\includegraphics*[width=0.49\columnwidth]{hap_300_2_10_5_1_1.56.eps}
\includegraphics*[width=0.5\columnwidth]{dist_300_2_10_5_1_1.56.eps}
$\eta=0.85\qquad n^*=10\qquad d=5$\\
\includegraphics*[width=0.49\columnwidth]{hap_300_2_10_5_1_0.85.eps}
\includegraphics*[width=0.5\columnwidth]{dist_300_2_10_5_1_0.85.eps}
$\eta=0.4\qquad n^*=30\qquad d=5$\\
\includegraphics*[width=0.49\columnwidth]{hap_300_2_30_5_0_0.4.eps}
\includegraphics*[width=0.5\columnwidth]{dist_300_2_30_5_0_0.4.eps}
$\eta=0.36\qquad n^*=50\qquad d=5$\\
\includegraphics*[width=0.49\columnwidth]{hap_300_2_50_5_0_0.36.eps}
\includegraphics*[width=0.5\columnwidth]{dist_300_2_50_5_0_0.36.eps}
\caption{\label{renormplotc}  Number of clusters and mean-error as a function of $s$ for $L1$-sphere 
(first and second rows) and $L2$-sphere (third and fourth rows) shaped cluster.}

\end{figure}
\begin{figure}[ht]
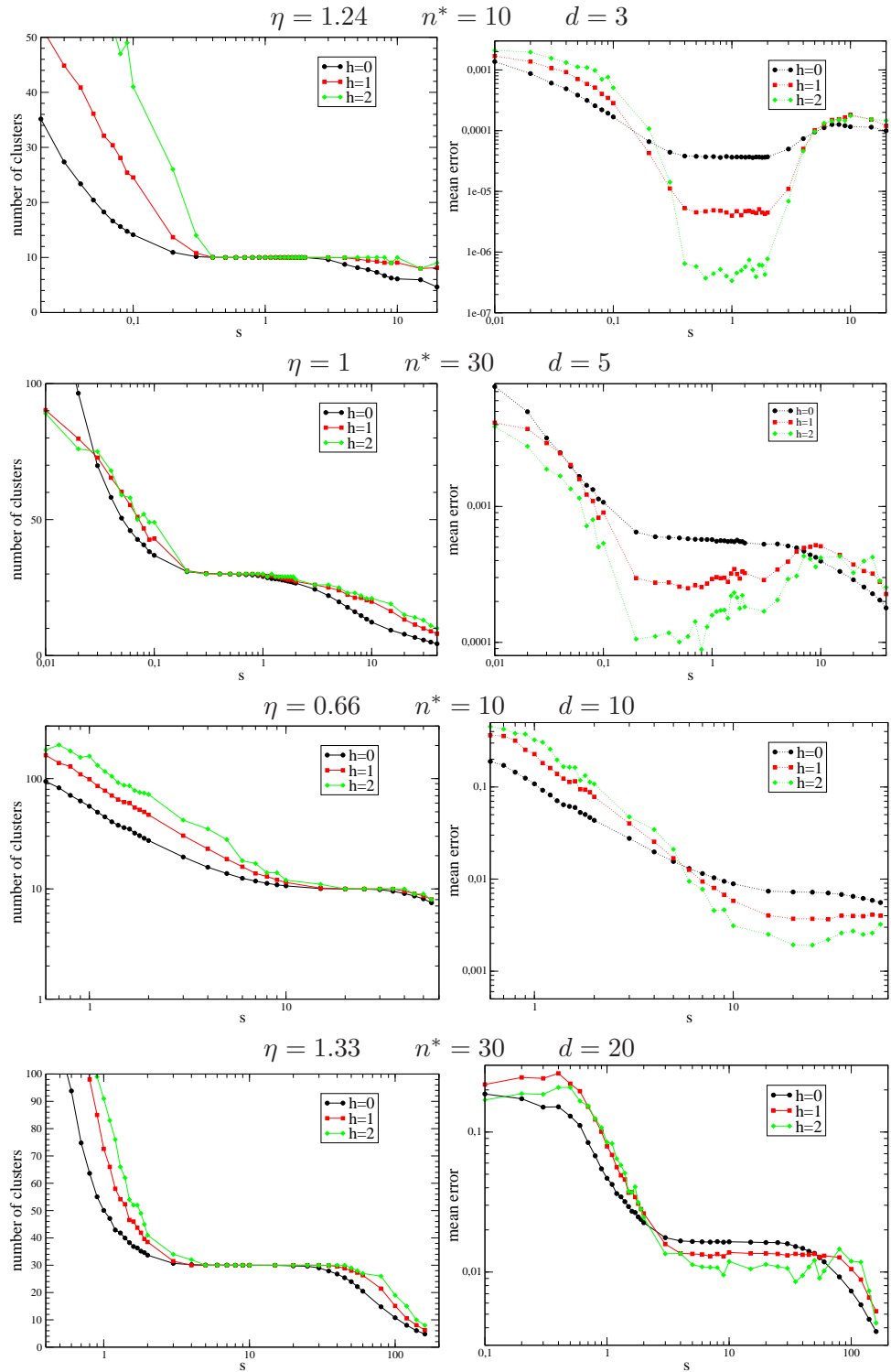

\centering
$\eta=1.24\qquad n^*=10\qquad d=3$\\
\includegraphics*[width=0.49\columnwidth]{hap_300_2_10_3_0_1.24.eps}
\includegraphics*[width=0.5\columnwidth]{dist_300_2_10_3_0_1.24.eps}
$\eta=1\qquad n^*=30\qquad d=5$\\
\includegraphics*[width=0.49\columnwidth]{hap_300_2_30_5_0_1.eps}
\includegraphics*[width=0.5\columnwidth]{dist_300_2_30_5_0_1.eps}
$\eta=0.66\qquad n^*=10\qquad d=10$\\
\includegraphics*[width=0.49\columnwidth]{hap_300_2_10_10_0_0.66.eps}
\includegraphics*[width=0.5\columnwidth]{dist_300_2_10_10_0_0.66.eps}
$\eta=1.33\qquad n^*=30\qquad d=20$\\
\includegraphics*[width=0.49\columnwidth]{hap_300_2_30_20_0_1.33.eps}
\includegraphics*[width=0.5\columnwidth]{dist_300_2_30_20_0_1.33.eps}
\caption{\label{renormplotd} 
Number of clusters and mean-error as a function of $s$ for $L2$-sphere shaped in $d=3,5,10,20$
with $n^*=10,30,10$.}
\end{figure}

Assume first that $x_1 = x^*$, which obtained if we set $s = s^*$ in the first clustering stage.
This means that  each cluster which 
is obtained at this stage is among the exact clusters with a reduced variance, 
resulting from the extreme value distribution properties (\ref{eq:rec})  
combined with definition (\ref{def:formfactor}) of the shape factor $\omega$:
\begin{equation}\label{sigmascale}
\sigma_c^{(\lambda)} = \frac{1}{\omega}\Bigl(\frac{\lambda N}{n_1}\Bigr)^{-2/d} \sigma_c = 
\frac{\lambda^{2/d}}{\omega}\sigma_c.
\end{equation}
Note at this point that 
\[
\frac{\omega}{\lambda^{2/d}}\gg 1,
\]
is required to be in the conditions of getting a cluster shaped by the extreme value distribution.
For $y>x^*$, the new distortion involves only the inner cluster distribution
of exemplars which is simply rescaled by this $(x_1/\lambda)^{2/d}$ factor, so from (\ref{eq:distor})
we conclude that
\[
\tilde\sigma_{x^*}^{(\lambda)}(y) = \sigma(y),\qquad\text{for}\qquad y\ge x^*. 
\]    
Instead, for $y<x^*$, the new distortion involves the merging of clusters, which inter distances,
contrary to their inner distances, are not rescaled and
are the same as in the original data set. This implies that
\[      
\frac{d\tilde\sigma_{x^*}^{(\lambda)}}{d y}(y) \le \sigma'(y),\qquad\text{for}\qquad y < x^*. 
\]    
As a result the optimal number of clusters is unchanged, $y_1 = x^*$.

For $x_1 <x^*$, which is obtained when $s>s^*$, the new distribution of data points, formed of exemplars,
is also governed by the extreme value distribution, and all cluster at this level are intrinsically 
true clusters, with a shape following the Weibull distribution. We are then necessary at the transition point
at this stage: $y^* = x_1$\footnote{Fluctuations are neglected in this argument. In practice the exemplars 
which emerge from the coalescence of two clusters might  originate from both clusters, 
when considering different subsets, if the number of data is not sufficiently large.}. 
In addition, the  cost of merging two clusters, i.e. $y$ slightly below $x_1$, is actually greater now after 
rescaling,
\[
\frac{d\tilde\sigma_{x_1}^{(\lambda)}}{d y}(y) \le \sigma'(y),\qquad\text{for}\qquad y = (x_1)_-, 
\]
because mutual cluster distances appear comparatively larger.
Instead, for $y$ slightly above $x_1$, the gain in distortion when $y$ increases is  smaller,
because it is due to the fragmentation of Weibull shaped cluster, as compared to the gain of
separating clusters in the coalescence phase et former level,
\[
\frac{d\tilde\sigma_{x_1}^{(\lambda)}}{d y}(y) \ge \sigma'(y),\qquad\text{for}\qquad y = (x_1)_+. 
\]
As a result, from the convexity property of $\sigma^{(\lambda)}(y)$, 
we then expect again that the solution of (\ref{eq:hap2})
remains unchanged $y_1= x_1$ in the second step with respect to the first one.

Finally, for $x_1 > x^*$, the new distribution of data points is not shaped by the extreme value
statistics when the number of fragmented clusters increases, because in that case the fragments 
are distributed in the entire volume of the fragmented cluster. In particular,
\[
\tilde\sigma_{x_1}^{(\lambda)}(y) \simeq 
\frac{\omega}{\lambda^{2/d}}\sigma(y),\qquad\text{when}\qquad x_1>>x^*. 
\]
The rescaling effect vanishes progressively when we get away from the transition point, 
so we conclude that the optimal density of clusters $y_1$ is displaced toward larger values in
this region.

We have tested this renormalized procedure, by generating artificial sets of datapoints 
in various situations, including different types of cluster shape. 
Some sample plots are displayed in Figure~\ref{renormplots} and \ref{renormplotc} to illustrate 
the preceding proposition~\ref{prop:renorm}. The self-similar point is clearly identified when plotting 
the number of clusters against the bare penalty, when $\eta$ is not to small. As expected from the scaling 
(\ref{penascale}), the effect is less sensible when the dimension increases, but remains perfectly visible 
and exploitable at least up to $d=30$. The absence of information loss
of the hierarchical procedure can be seen on the mean-error plots, in the region of $s$  
around the critical value $s^*$. The results are stable, when we take 
into account at the first stage of the hierarchical procedure the influence of the shape of the clusters.
This is done by fixing the value of the factor form $\omega$ to the correct value. In that case, at subsequent 
levels of the hierarchy the default value $\omega=1$ is the correct one to give consistent results.
Nevertheless if the factor form is unknown and set to false default value, 
the results are spoiled at subsequent levels, and the
underlying  number of clusters turns out  to be more difficult to identify, depending on the discrepancy of  
$\omega$ with respect to its default value. We see also that the identification of the 
transition point is still possible when the number of datapoints per cluster 
get smaller (down to $6$ in these tests). 

\section{Discussion and perspectives}
The present analysis of the scaling properties of AP, within a divide-and-conquer setting gives us
a simple way to identify a self-similar property of the special point $s^*$, for which the exact structure
of the clusters is recovered. This property can be actually exploited, when the dimension is 
not too large and when the clusters are sufficiently far apart and sufficiently populated. 
The separability property is actually controlled  
by the parameter $\eta$ introduced in \ref{def:eta}. For $\eta\ge1$ the underlying cluster structure is recovered, 
and in the vicinity of $s^*$, the absence of information loss, deduced from  the single cluster analysis is 
effectively observed. The search of the exact number of cluster could then be turned into a simple 
line-search algorithm combined with \RAP. This deserves further exploration, in particular from the application point
of view, on real data and for clusters of unknown shape.  
From the theoretical viewpoint, this renormalization approach to the self-tuning of algorithm parameter could 
be applied in other context, where self-similarity is a key property at large scale. First it is not yet clear 
how we could adapt \RAP\ to the \SCAP\ context.  The principal component 
analysis and associated spectral clustering provide other examples, where the fixing of the number of selected 
components is usually not obtained by some self-consistent procedure and where a similar approach to the 
one presently proposed could be used.

\paragraph{Acknowledgments} This work was supported by the French National Research Agency (ANR) grant N∞ ANR-08-SYSC-017.

{\small \bibliography{refer}}
\bibliographystyle{unsrt}

\newpage
\appendix
\section{Proof of Proposition 2.2}\label{proofprop}
The influence between the center of mass
and extreme value statistics distribution corresponds to corrections which vanish when $M$ tends to 
infinity (see Appendix~\ref{corrections}. Neglecting these corrections, enables us to use
a spherical kernel instead of cylindrical kernel and to making no distinction between $ex$ and $\tilde ex$,
to write the recurrence. Between level $h$ and $h+1$, one has:
\begin{equation}\label{eq:iter}
f_{sd}^{(h+1)}(x) = \int_0^\infty K^{(h,M)}(x,y)f_{ex}^{(h,M)}(y)dy
\end{equation}
with \vspace{-.07in}
\begin{equation}\label{kernel}
\lim_{M\to\infty} M^{-1} K^{(h,M)}(\frac{x}{M},\frac{y}{M})=\frac{d}{\sigma^{(h)}}K(\frac{d
  x}{\sigma^{(h)}},\frac{d y}{\sigma^{(h)}})
\end{equation}
where $K(x,y)$ is the d-dimensional radial diffusion kernel,
\[
K(x,y)\egaldef
\frac{1}{2}x^{\frac{d-2}{4}}y^{\frac{2-d}{4}}I_{\frac{d-2}{2}}\bigl(
\sqrt{xy}\bigr)e^{-\frac{x+y}{2}}.
\]
with $I_{\frac{d}{2}-1}$ the modified Bessel function of index
${d/2-1}$. The selection
mechanism of the exemplar yields at level $h$,
\[
F_{ex}^{(h,M)}(x) = \bigl(F_{sd}^{(h)}(x)\bigr)^M,
\]
and with a by part integration, (\ref{eq:iter}) rewrites as:
\begin{eqnarray*}
f_{sd}^{(h+1)}(x) &= & K^{(h,M)}(x,0) +  \\
   &  & \int_0^\infty
\bigl(F_{sd}^{(h)}(y)\bigr)^M\ \frac{\partial
  K^{(h,M)}}{\partial y}(x,y)dy,
\end{eqnarray*}\vspace{-.1in}
with \vspace{-.05in}
\begin{align*}
\lim_{M\to\infty}M^{-1} &K^{(h,M)}(\frac{x}{M},0) =\\
&\frac{d}{2\Gamma(\frac{d}{2})\sigma^{(h)}}\bigl(\frac{d
x}{2\sigma^{(h)}}\bigr)^{\frac{d}{2}-1} \exp\bigl(-\frac{d
x}{2\sigma^{(h)}}\bigr).
\end{align*}
At this point the recursive
hierarchical clustering is described as a closed form equation. Proposition \ref{scaling} is then based on
(\ref{kernel}) and on the following scaling behaviors,
\[
\lim_{M\to\infty}
F_{ex}^{(h,M)}\bigl(\frac{x}{M^{\frac{2}{d}}}\bigr) =
\exp\bigl(-\alpha^{(h)}x^{\frac{d}{2}}\bigr),
\]
so that
\begin{align*}
\lim_{M\to\infty}F_{sd}^{(h+1)}(\frac{x}{M^\gamma}) &=
\lim_{M\to\infty} M^{1-\gamma}\int_0^\infty dy\\
\int_{\frac{x}{\sigma^{(h)}}}^\infty du
&f_{ex}^{(h,M)}(\frac{y}{M^\frac{2}{d}})
K(M^{1-\gamma}u,\frac{M^{1-\frac{d}{2}}y}{\sigma^{(h)}}).
\end{align*}
Basic asymptotic properties of $I_{d/2-1}$ yield with a proper choice of $\gamma$,
the non degenerate limits  of proposition~\ref{scaling}.
In the particular case $d=2$, taking $\gamma=1$, it comes:
\[
\lim_{M\to\infty}F_{sd}^{(h+1)}(\frac{x}{M}) = \int_0^\infty
dy\int_{\frac{x}{\sigma^{(h)}}}^\infty du
f_{ex}^{(h)}(\sigma^{(h)}y) K(u,y)\\[0.2cm]
\]\vspace{-.2in}
\[
 \quad = -\int_0^\infty dy\int_{\frac{x}{\sigma^{(h)}}}^\infty du
\frac{d e^{-\alpha^{(h)}\sigma^{(h)}x}}{dy}I_0(2\sqrt{uy}) e^{-(u+y)}\\[0.2cm]
\]\vspace{-.2in}
\[
=\exp\bigl(-\frac{\alpha^{(h)}}{1+\alpha^{(h)}\sigma^{(h)}}x\bigr), \qquad \qquad \qquad \qquad
\]
with help of the identity
\[
\int_0^\infty dx x^\nu e^{-\alpha x}I_{2\nu}(2\beta\sqrt{x}) =
\frac{1}{\alpha}\bigl(\frac{\beta}{\alpha}\bigr)^{2\nu}e^{\frac{\beta}{\alpha}}.
\]
Again in the particular case $d=2$, by virtue of the exponential law one further has
 $\alpha^{(h)} = 1/\sigma^{(h)}$, finally yielding:
\begin{equation}\label{beta}
\beta^{(h+1)} = \frac{1}{2}\beta^{(h)}.
\end{equation}

\section{Finite size corrections}\label{corrections}
We consider a given hierarchical level $h$, 
$\r$ denotes sample points, $\rcm$ their corresponding center of mass, and 
$\rc$ the exemplar, which in turn becomes a sample point at level $h+1$.
We have
\begin{align*}
p_{sd}^{(h+1)}(\r)d^d\r &= P(\rc\in d^d\r)\\[0.2cm]  
&= d^d\r\int d^d\rcm p_{sd,cm}^{(h)}(\r,\rcm)
P(\vert \rsd - \rcm\vert\ge \vert\r-\rcm\vert\big\vert\rcm)^{M-1}.
\end{align*}
We analyse this equation with the help of a generating function: 
\[
\phi_{\centerdot}(\Lambda) = \int d^d\r p_{\centerdot}(\r)e^{- \Lambda\r}.
\]
where $\centerdot$ may be indifferently  $sd$, $c$ or $cm$ and $\Lambda\r$ is the ordinary 
scalar product between two $d$-dimensional vectors. Let $\lambda = \vert\Lambda\vert$,
by rotational invariance, $p_\centerdot$ depends only on $r$ and $\phi_{\centerdot}$ depends solely on $\lambda$,
so we have
\[
g_\centerdot(\lambda)  \egaldef \log(\phi_\centerdot(\Lambda)) = 
\log\Bigl(2\pi^{d/2}\int_0^\infty dr r^{d-1}p_{\centerdot}(r)
\bigl(\frac{\lambda r}{2}\bigr)^{1-d/2}I_{d/2-1}(\lambda r)\Bigr).
\]
The joint distribution between $\r_{sd}$ and $\rcm$ takes the following form
\[
p_{sd,cm}(\r,\rcm) = p_{sd}(r) p_{cm\vert sd}(\vert\rcm-\frac{\r}{M}\vert)
\]
where by definition $p_{cm\vert sd}$ is the conditional density of $\rcm$ to $\r_{sd}$, 
with
\begin{equation}\label{gcmsd}
g_{cm\vert sd}(\lambda) = (M-1)g_{sd}\bigl(\frac{\lambda}{M}\bigr),
\end{equation}
while 
\begin{equation}\label{gcm}
g_{cm}(\lambda) = Mg_{sd}\bigl(\frac{\lambda}{M}\bigr),
\end{equation}
where $g_{sd}$ is assumed to have a non zero radius Taylor expansion of the 
form
\begin{equation}\label{gsd}
g_{sd}(\lambda) = \frac{\sigma^{(h)}}{2d}\lambda^2 + \sum_{n=2}^\infty \frac{g^{(2n)}(0)}{2n!}\lambda^{2n},
\end{equation}
since by rotational symmetry all odd powers of $\lambda$ vanish and where $\sigma^{(h)}$ represents the variance at level
$h$ of the sample data distribution.
In addition the conditional probability density of $\r_{sd}$ to $\rcm$ reads 
\[
p_{sd\vert cm} (\r,\rcm) = \frac{p_{sd}(r)}{p_{cm}(r_{cm})} p_{cm\vert sd}(\vert\rcm-\frac{\r}{M}\vert)
\egaldef p_{sd\vert cm}(u,\theta,r_{cm})
\]
where ${\bf u} =\r - \rcm$ and $\theta$ is the angle between ${\bf u}$ and $\rcm$. Let
\[
f(u,r_{cm}) \egaldef P(\vert \rsd - \rcm\vert\ge u \big\vert\rcm).
\]
We have
\[
f(u,r_{cm}) = 1- \Omega_{d-1}\int_0^u dx x^{d-1}\int_0^{\pi}d\theta \sin\theta^{d-2} p_{sd\vert cm}(x,\theta,r_{cm}).
\]
with
\[
\Omega_d = \frac{2\pi^{d/2}}{\Gamma\bigl(\frac{d}{2}\bigr)},
\]
the $d$-dimensional solid angle.
Let 
\[
h(u,r_{cm}) \egaldef \log (f(u,r_{cm})).
\]
We have
\[
p_{sd}^{(h+1)}(r) = p_{sd}^{(h)}(r) \int d^d\rcm p_{cm\vert sd}(\vert\rcm-\frac{\r}{M}\vert)   
\exp \bigl((M-1)h(\vert {\bf r}-\rcm\vert,r_{cm})\bigr).
\]
From the expansion (\ref{gsd}) we see that corrections in $g_{cm}$ and 
$g_{cm\vert sd}$ to the Gaussian distribution 
are of order $1/M^3$, 
$\sigma_{cm} = \sigma/M$ as expected from the central limit theorem and 
$\sigma_{cm\vert sd} = (M-1)\sigma/M^2$. Letting ${\bf} y = {\bf r}_{cm}-{\bf r}$ 
we have
\[
p_{sd}^{(h+1)}(r) = 
p_{sd}^{(h)}(0)\Bigl(\frac{dM}{2\pi\sigma^{(h)}}\Bigr)^{d/2}\int d^d{\bf y}
\ \exp\Bigl(-M\psi^{(M)}({\bf r},{\bf y})\Bigr),  
\]
with 
\begin{align*}
\psi^{(M)}({\bf r},{\bf y})
&\egaldef -\frac{d}{2}\log\frac{M}{M-1}-\frac{dr^2}{2\sigma^{(h)}}+
\log\frac{p_{sd}^{(h)}(r)}{p_{sd}^{(h)}(0)}\\[0.2cm]
&+\frac{d M}{2(M-1)\sigma^{(h)}} \vert{\bf y + r }\vert^2
+(M-1)h\bigl( y,\vert {\bf y +r}\vert \bigr).
\end{align*}
As observed previously $p_{sd}^{(h+1)}(r/M^{1/d})$ converges to a  Weibull  distribution when
$M$ goes to infinity, and the corrections to this are obtained with help of 
the following approximation:
\[
\psi^{(M)}(\frac{\bf r}{M^{1/d}},{\bf y}) = \frac{d}{2\sigma^{(h)}}\vert{\bf y}+\frac{\bf r }{M^{1/d}}\vert^2 +
\alpha^{(h)}y^d
+ O\bigl(\frac{1}{M}\bigr),\\[0.2cm]
\]
with 
\[
\alpha^{(h)} = p_{sd}^{(h)}(0)\frac{\Omega_d}{d}.
\]
As a result, 
computing the normalization constant $p_{sd}^{(h+1)}(0)$  and the corresponding 
variance $\sigma^{(h+1)}$, yields the following recurrence relations:
\[
\begin{cases}
\DD \alpha^{(h+1)} =  \alpha^{(h)} 
\ +\ O\Bigl(\frac{1}{M}\Bigr).\\[0.3cm]
\DD \sigma^{(h+1)}  = \Gamma\bigl(1+\frac{2}{d}\bigr){\alpha^{(h)}}^{-2/d}\Bigl(1+\frac{\sigma^{(h)}\alpha^{2/d}}
{\Gb}\frac{1}{M^{1-2/d}}\Bigr) 
+o\bigl(M^{2/d-1}\bigr).
\end{cases}
\]
Letting 
\[
\omega^{(h)} \egaldef \frac{\sigma^{(h)}{\alpha^{(h)}}^{2/d}}{\Gamma\bigl(1+\frac{2}{d}\bigr)},
\]
we get
\[
\omega^{(h+1)} = 1 + \frac{\omega^{(h)}}{M^{1-2/d}} +o\bigl(M^{2/d-1}\bigr).
\]
Consequently, for $h=0$, we have
\[
\sigma^{(1)} = \frac{\sigma^{(0)}}{\omega^{(0)}}\Bigl(1+\frac{\omega^{(0)}}{M^{1-2/d}}\Bigr)
+o\bigl(M^{2/d-1}\bigr),
\]
while for $h>1$ we get
\[
\sigma^{(h+1)} = \sigma^{(h)}\Bigl(1+\frac{\omega^{(h)}-\omega^{(h-1)}}{M^{1-2/d}}\Bigr)
+o\bigl(M^{2/d-1}\bigr).
\]
For $h=1$ this reads
\[
\sigma^{(2)} = \sigma^{(1)}\Bigl(1+\frac{1-\omega^{(0)}}{M^{1-2/d}}\Bigr)
+o\bigl(M^{2/d-1}\bigr),
\]
and thereby
\[
\sigma^{(h+1)} = \sigma^{(h)}
+o\bigl(M^{2/d-1}\bigr),\qquad\text{for}\ h>1.
\]

\end{document}

%% file: factor_graph.pstex_t
\begin{picture}(0,0)%
\includegraphics{factor_graph.pstex}%
\end{picture}%
\setlength{\unitlength}{3947sp}%
\begingroup\makeatletter\ifx\SetFigFont\undefined%
\gdef\SetFigFont#1#2#3#4#5{%
  \reset@font\fontsize{#1}{#2pt}%
  \fontfamily{#3}\fontseries{#4}\fontshape{#5}%
  \selectfont}%
\fi\endgroup%
\begin{picture}(7400,4819)(2243,-5398)
\put(3931,-2191){\makebox(0,0)[lb]{\smash{{\SetFigFont{20}{24.0}{\rmdefault}{\mddefault}{\updefault}a}}}}
\put(3031,-5191){\makebox(0,0)[lb]{\smash{{\SetFigFont{20}{24.0}{\rmdefault}{\mddefault}{\updefault}b}}}}
\put(7531,-3406){\makebox(0,0)[lb]{\smash{{\SetFigFont{20}{24.0}{\rmdefault}{\mddefault}{\updefault}c}}}}
\put(2371,-1216){\makebox(0,0)[lb]{\smash{{\SetFigFont{20}{24.0}{\rmdefault}{\mddefault}{\updefault}$1$}}}}
\put(5386,-991){\makebox(0,0)[lb]{\smash{{\SetFigFont{20}{24.0}{\rmdefault}{\mddefault}{\updefault}$2$}}}}
\put(5746,-3016){\makebox(0,0)[lb]{\smash{{\SetFigFont{20}{24.0}{\rmdefault}{\mddefault}{\updefault}$3$}}}}
\put(4321,-4396){\makebox(0,0)[lb]{\smash{{\SetFigFont{20}{24.0}{\rmdefault}{\mddefault}{\updefault}$4$}}}}
\put(9136,-3991){\makebox(0,0)[lb]{\smash{{\SetFigFont{20}{24.0}{\rmdefault}{\mddefault}{\updefault}$6$}}}}
\put(9106,-2416){\makebox(0,0)[lb]{\smash{{\SetFigFont{20}{24.0}{\rmdefault}{\mddefault}{\updefault}$5$}}}}
\end{picture}%

%% file: APfactorgraph.pstex_t
\begin{picture}(0,0)%
\includegraphics{APfactorgraph.pstex}%
\end{picture}%
\setlength{\unitlength}{3947sp}%
\begingroup\makeatletter\ifx\SetFigFont\undefined%
\gdef\SetFigFont#1#2#3#4#5{%
  \reset@font\fontsize{#1}{#2pt}%
  \fontfamily{#3}\fontseries{#4}\fontshape{#5}%
  \selectfont}%
\fi\endgroup%
\begin{picture}(5424,3988)(2839,-4648)
\put(3526,-2011){\makebox(0,0)[lb]{\smash{{\SetFigFont{25}{30.0}{\rmdefault}{\mddefault}{\updefault}$a_{\mu\to i}$}}}}
\put(4756,-1456){\makebox(0,0)[lb]{\smash{{\SetFigFont{25}{30.0}{\rmdefault}{\mddefault}{\updefault}$S(i,c_i)$}}}}
\put(5056,-2791){\makebox(0,0)[lb]{\smash{{\SetFigFont{25}{30.0}{\rmdefault}{\mddefault}{\updefault}$c_i$}}}}
\put(6196,-3121){\makebox(0,0)[lb]{\smash{{\SetFigFont{25}{30.0}{\rmdefault}{\mddefault}{\updefault}$r_{i\to \nu}$}}}}
\put(7756,-2821){\makebox(0,0)[lb]{\smash{{\SetFigFont{25}{30.0}{\rmdefault}{\mddefault}{\updefault}$\chi_\nu$}}}}
\put(2926,-1081){\makebox(0,0)[lb]{\smash{{\SetFigFont{25}{30.0}{\rmdefault}{\mddefault}{\updefault}$\chi_\mu$}}}}
\end{picture}%

%% file: hap.pstex_t
\begin{picture}(0,0)%
\includegraphics{hap.pstex}%
\end{picture}%
\setlength{\unitlength}{3947sp}%
\begingroup\makeatletter\ifx\SetFigFont\undefined%
\gdef\SetFigFont#1#2#3#4#5{%
  \reset@font\fontsize{#1}{#2pt}%
  \fontfamily{#3}\fontseries{#4}\fontshape{#5}%
  \selectfont}%
\fi\endgroup%
\begin{picture}(10222,5889)(1516,-6368)
\put(6957,-749){\makebox(0,0)[lb]{\smash{{\SetFigFont{20}{24.0}{\rmdefault}{\mddefault}{\updefault}Dataset}}}}
\put(5281,-3095){\makebox(0,0)[lb]{\smash{{\SetFigFont{20}{24.0}{\rmdefault}{\mddefault}{\updefault}{\color{blue}WAP}}}}}
\put(6398,-5218){\makebox(0,0)[lb]{\smash{{\SetFigFont{20}{24.0}{\rmdefault}{\mddefault}{\updefault}{\color{blue}WAP}}}}}
\put(9806,-6000){\makebox(0,0)[lb]{\smash{{\SetFigFont{20}{24.0}{\rmdefault}{\mddefault}{\updefault}{\color{red}Exemplars}}}}}
\put(8130,-6223){\makebox(0,0)[lb]{\smash{{\SetFigFont{20}{24.0}{\rmdefault}{\mddefault}{\updefault}{\color{blue}WAP}}}}}
\put(1531,-1816){\makebox(0,0)[lb]{\smash{{\SetFigFont{20}{24.0}{\rmdefault}{\mddefault}{\updefault}$h=0$}}}}
\put(1576,-4276){\makebox(0,0)[lb]{\smash{{\SetFigFont{20}{24.0}{\rmdefault}{\mddefault}{\updefault}$h=1$}}}}
\put(1621,-6151){\makebox(0,0)[lb]{\smash{{\SetFigFont{20}{24.0}{\rmdefault}{\mddefault}{\updefault}$h=2$}}}}
\end{picture}%

%% file: renorm.pstex_t
\begin{picture}(0,0)%
\includegraphics{renorm.pstex}%
\end{picture}%
\setlength{\unitlength}{3947sp}%
\begingroup\makeatletter\ifx\SetFigFont\undefined%
\gdef\SetFigFont#1#2#3#4#5{%
  \reset@font\fontsize{#1}{#2pt}%
  \fontfamily{#3}\fontseries{#4}\fontshape{#5}%
  \selectfont}%
\fi\endgroup%
\begin{picture}(10888,7979)(76,-7266)
\put(3436,482){\makebox(0,0)[lb]{\smash{{\SetFigFont{20}{24.0}{\rmdefault}{\mddefault}{\updefault}$M$ data subsets}}}}
\put(4576,-268){\makebox(0,0)[lb]{\smash{{\SetFigFont{14}{16.8}{\rmdefault}{\mddefault}{\updefault}$\lambda N$}}}}
\put(3991,-1078){\makebox(0,0)[lb]{\smash{{\SetFigFont{14}{16.8}{\rmdefault}{\mddefault}{\updefault}$\lambda N$}}}}
\put(5371,-1063){\makebox(0,0)[lb]{\smash{{\SetFigFont{14}{16.8}{\rmdefault}{\mddefault}{\updefault}$\lambda N$}}}}
\put(4696,-1858){\makebox(0,0)[lb]{\smash{{\SetFigFont{14}{16.8}{\rmdefault}{\mddefault}{\updefault}$\lambda N$}}}}
\put(8356,-6721){\makebox(0,0)[lb]{\smash{{\SetFigFont{20}{24.0}{\rmdefault}{\mddefault}{\updefault}$n_2$}}}}
\put(676,-4171){\makebox(0,0)[lb]{\smash{{\SetFigFont{20}{24.0}{\rmdefault}{\mddefault}{\updefault}$N$  data points}}}}
\put(5326,-4561){\makebox(0,0)[lb]{\smash{{\SetFigFont{25}{30.0}{\rmdefault}{\mddefault}{\updefault}$\lambda N s$}}}}
\put(9376,-1336){\makebox(0,0)[lb]{\smash{{\SetFigFont{20}{24.0}{\rmdefault}{\mddefault}{\updefault}$M n_1 = \lambda N$}}}}
\put(9451,-4111){\makebox(0,0)[lb]{\smash{{\SetFigFont{25}{30.0}{\rmdefault}{\mddefault}{\updefault}$\lambda N s(\lambda)$}}}}
\put(7126,-886){\makebox(0,0)[lb]{\smash{{\SetFigFont{25}{30.0}{\rmdefault}{\mddefault}{\updefault}HAP Step $1$}}}}
\put(9676,-3136){\makebox(0,0)[lb]{\smash{{\SetFigFont{25}{30.0}{\rmdefault}{\mddefault}{\updefault}HAP Step $2$}}}}
\put(7276,-1711){\makebox(0,0)[lb]{\smash{{\SetFigFont{25}{30.0}{\rmdefault}{\mddefault}{\updefault}$\lambda N s$}}}}
\end{picture}%

%% file: cluster1.pstex_t
\begin{picture}(0,0)%
\includegraphics{cluster1.pstex}%
\end{picture}%
\setlength{\unitlength}{3947sp}%
\begingroup\makeatletter\ifx\SetFigFont\undefined%
\gdef\SetFigFont#1#2#3#4#5{%
  \reset@font\fontsize{#1}{#2pt}%
  \fontfamily{#3}\fontseries{#4}\fontshape{#5}%
  \selectfont}%
\fi\endgroup%
\begin{picture}(16857,10279)(128,-9669)
\put(5461,-3631){\makebox(0,0)[lb]{\smash{{\SetFigFont{20}{24.0}{\rmdefault}{\mddefault}{\updefault}$s \sim s^*$}}}}
\put(5656,-1861){\makebox(0,0)[lb]{\smash{{\SetFigFont{20}{24.0}{\rmdefault}{\mddefault}{\updefault}$s\ll s^*$}}}}
\put(5551,-6316){\makebox(0,0)[lb]{\smash{{\SetFigFont{20}{24.0}{\rmdefault}{\mddefault}{\updefault}$s\gg s^*$}}}}
\end{picture}%

%% file: sigmadx.pstex_t
\begin{picture}(0,0)%
\includegraphics{sigmadx.pstex}%
\end{picture}%
\setlength{\unitlength}{3947sp}%
\begingroup\makeatletter\ifx\SetFigFont\undefined%
\gdef\SetFigFont#1#2#3#4#5{%
  \reset@font\fontsize{#1}{#2pt}%
  \fontfamily{#3}\fontseries{#4}\fontshape{#5}%
  \selectfont}%
\fi\endgroup%
\begin{picture}(9894,6249)(1909,-5698)
\put(10651,-5236){\makebox(0,0)[lb]{\smash{{\SetFigFont{20}{24.0}{\rmdefault}{\mddefault}{\updefault}$x$}}}}
\put(4651,-5311){\makebox(0,0)[lb]{\smash{{\SetFigFont{20}{24.0}{\rmdefault}{\mddefault}{\updefault}$x^*$}}}}
\put(2701,-5161){\makebox(0,0)[lb]{\smash{{\SetFigFont{20}{24.0}{\rmdefault}{\mddefault}{\updefault}Coalescence}}}}
\put(6826,-5161){\makebox(0,0)[lb]{\smash{{\SetFigFont{20}{24.0}{\rmdefault}{\mddefault}{\updefault}Fragmentation}}}}
\put(8356,-826){\makebox(0,0)[lb]{\smash{{\SetFigFont{20}{24.0}{\rmdefault}{\mddefault}{\updefault}$\sigma(x)$}}}}
\put(8356,-1306){\makebox(0,0)[lb]{\smash{{\SetFigFont{20}{24.0}{\rmdefault}{\mddefault}{\updefault}$\sigma_{x_1}^{(\lambda)}(x)$}}}}
\put(5461,-2506){\makebox(0,0)[lb]{\smash{{\SetFigFont{14}{16.8}{\rmdefault}{\mddefault}{\updefault}Fragmentation}}}}
\put(3331,-1141){\makebox(0,0)[lb]{\smash{{\SetFigFont{14}{16.8}{\rmdefault}{\mddefault}{\updefault}Coalescence}}}}
\end{picture}%

%% file: ndx.pstex_t
\begin{picture}(0,0)%
\includegraphics{ndx.pstex}%
\end{picture}%
\setlength{\unitlength}{3947sp}%
\begingroup\makeatletter\ifx\SetFigFont\undefined%
\gdef\SetFigFont#1#2#3#4#5{%
  \reset@font\fontsize{#1}{#2pt}%
  \fontfamily{#3}\fontseries{#4}\fontshape{#5}%
  \selectfont}%
\fi\endgroup%
\begin{picture}(8124,5424)(1639,-6193)
\put(8641,-5986){\makebox(0,0)[lb]{\smash{{\SetFigFont{20}{24.0}{\rmdefault}{\mddefault}{\updefault}$s$}}}}
\put(1786,-1366){\makebox(0,0)[lb]{\smash{{\SetFigFont{20}{24.0}{\rmdefault}{\mddefault}{\updefault}$n(s)$}}}}
\put(6601,-2176){\makebox(0,0)[lb]{\smash{{\SetFigFont{20}{24.0}{\rmdefault}{\mddefault}{\updefault}$h=0$}}}}
\put(6631,-2656){\makebox(0,0)[lb]{\smash{{\SetFigFont{20}{24.0}{\rmdefault}{\mddefault}{\updefault}$h=1$}}}}
\put(4801,-5986){\makebox(0,0)[lb]{\smash{{\SetFigFont{20}{24.0}{\rmdefault}{\mddefault}{\updefault}$s^*$}}}}
\end{picture}%